\documentclass[11pt]{article}
\usepackage{latexsym,amssymb,amsmath} 
\usepackage{path}
\usepackage{url}
\usepackage{verbatim}

\usepackage{natbib}
\usepackage{amsfonts}
\usepackage{amsfonts}
\usepackage{amsfonts}
\usepackage{array}
\usepackage{mathrsfs}
\usepackage{amsfonts}
\usepackage{amsmath}
\usepackage{amssymb}
\usepackage{amsthm}
\usepackage{bm}
\usepackage{appendix}

\usepackage{graphicx} 
\usepackage{subfigure}

\graphicspath{{figures/}}

\usepackage[lined,boxed,ruled, commentsnumbered]{algorithm2e}

\newtheorem{lemma}{Lemma}
\newtheorem{theorem}{Theorem}
\newtheorem{corollary}{Corollary}

\newtheorem{definition}{Definition}
\newtheorem{remark}{Remark}

\newcommand{\supp}{\text{supp}}

\newcommand{\argmin}{\mathop{\arg\min}}
\newcommand{\argmax}{\mathop{\arg\max}}


\numberwithin{equation}{section}
\numberwithin{table}{section}
\numberwithin{figure}{section}

\newcommand{\junk}[1]{{}}

\newlength{\fwtwo} 

\DeclareGraphicsExtensions{.jpg}

\title{Dual Iterative Hard Thresholding: From Non-convex Sparse Minimization to Non-smooth Concave Maximization}
\author{
  Bo Liu$^{1}$,\ \  Xiao-Tong-Yuan$^{2}$, \ \ Lezi Wang$^{1}$ \ \ Qingshan Liu$^{2}$ \ \ Dimitris N. Metaxas$^{2}$\\\\
  1. Department of Computer Science, Rutgers University \\
  \\
  \and
  2. B-DAT Lab, Nanjing University of Information Science \& Technology \\
 }
\date{}

\begin{document}

\maketitle

\begin{abstract}
Iterative Hard Thresholding (IHT) is a class of projected gradient descent methods for optimizing sparsity-constrained minimization models, with the best known efficiency and scalability in practice. As far as we know, the existing IHT-style methods are designed for sparse minimization in primal form. It remains open to explore duality theory and algorithms in such a non-convex and NP-hard problem setting. In this paper, we bridge this gap by establishing a duality theory for sparsity-constrained minimization with $\ell_2$-regularized loss function and proposing an IHT-style algorithm for dual maximization. Our sparse duality theory provides a set of sufficient and necessary conditions under which the original NP-hard/non-convex problem can be equivalently solved in a dual formulation. The proposed dual IHT algorithm is a super-gradient method for maximizing the non-smooth dual objective. An interesting finding is that the sparse recovery performance of dual IHT is invariant to the Restricted Isometry Property (RIP), which is required by virtually all the existing primal IHT algorithms without sparsity relaxation. Moreover, a stochastic variant of dual IHT is proposed for large-scale stochastic optimization. Numerical results demonstrate the superiority of dual IHT algorithms to the state-of-the-art primal IHT-style algorithms in model estimation accuracy and computational efficiency.
\end{abstract}

\subparagraph{Key words.} Distributed Optimization, Structured-Sparsity, ADMM, Image Segmentation.


\section{Introduction}
Sparse learning has emerged as an effective approach to alleviate model overfitting when feature dimension outnumbers training sample. Given a set of training samples$\{(x_i, y_i)\}_{i=1}^N$ in which $x_i\in \mathbb{R}^d$ is the feature representation and $y_i\in \mathbb{R}$ the corresponding label, the following sparsity-constrained $\ell_2$-norm regularized loss minimization problem is often considered in high-dimensional analysis:
\begin{equation}\label{prob:general}
\min_{\|w\|_0 \le k} P(w):= \frac{1}{N}\sum\limits_{i=1}^N l(w^\top x_i,y_i) + \frac{\lambda}{2} \|w\|^2.
\end{equation}
Here $l(\cdot;\cdot)$ is a convex loss function, $w\in \mathbb{R}^d$ is the model parameter vector and $\lambda$ controls the regularization strength. For example, the squared loss $l(a,b)=(b-a)^2$ is used in linear regression and the hinge loss $l(a,b)=\max\{0,1-ab\}$ in support vector machines. Due to the presence of cardinality constraint $\|w\|_0\le k$, the problem~\eqref{prob:general} is simultaneously non-convex and NP-hard in general, and thus is challenging for optimization. A popular way to address this challenge is to use proper convex relaxation, e.g., $\ell_1$ norm~\citep{tibshirani1996regression} and $k$-support norm~\citep{argyriou2012sparse}, as an alternative of the cardinality constraint. However, the convex relaxation based techniques tend to introduce bias for parameter estimation.

In this paper, we are interested in algorithms that directly minimize the non-convex formulation in~\eqref{prob:general}. Early efforts mainly lie in compressed sensing for signal recovery, which is a special case of~\eqref{prob:general} with squared loss. Among others, a family of the so called Iterative Hard Thresholding (IHT) methods~\citep{blumensath2009iterative,foucart2011hard} have gained significant interests and they have been witnessed to offer the fastest and most scalable solutions in many cases. More recently, IHT-style methods have been generalized to handle generic convex loss functions~\citep{beck2013sparsity,Yuan-ICML-2014,jain2014iterative} as well as structured sparsity constraints~\citep{jain2016structured}. The common theme of these methods is to iterate between gradient descent and hard thresholding to maintain sparsity of solution while minimizing the objective value.

Although IHT-style methods have been extensively studied, the state-of-the-art is only designed for the primal formulation~\eqref{prob:general}. It remains an open problem to investigate the feasibility of solving the original NP-hard/non-convex formulation in a dual space that might potentially further improve computational efficiency. To fill this gap, inspired by the recent success of dual methods in regularized learning problems, we systematically build a sparse duality theory and propose an IHT-style algorithm along with its stochastic variant for dual optimization.

\noindent\textbf{Overview of our contribution.} The core contribution of this work is two-fold in theory and algorithm. As the theoretical contribution, we have established a novel sparse Lagrangian duality theory for the NP-hard/non-convex problem~\eqref{prob:general} which to the best of our knowledge has not been reported elsewhere in literature. We provide in this part a set of \emph{sufficient and necessary} conditions under which one can safely solve the original non-convex problem through maximizing its concave dual objective function. As the algorithmic contribution, we propose the dual IHT (DIHT) algorithm as a super-gradient method to maximize the non-smooth dual objective. In high level description, DIHT iterates between dual gradient ascent and primal hard thresholding pursuit until convergence. A stochastic variant of DIHT is proposed to handle large-scale learning problems. For both algorithms, we provide non-asymptotic convergence analysis on parameter estimation error, sparsity recovery, and primal-dual gap as well. In sharp contrast to the existing analysis for primal IHT-style algorithms, our analysis is not relying on Restricted Isometry Property (RIP) conditions and thus is less restrictive in real-life high-dimensional statistical settings. Numerical results on synthetic datasets and machine learning benchmark datasets demonstrate that dual IHT significantly outperforms the state-of-the-art primal IHT algorithms in accuracy and efficiency. The theoretical and algorithmic contributions of this paper are highlighted in below:
\begin{itemize}
  \item Sparse Lagrangian duality theory: we established a sparse saddle point theorem (Theorem~\ref{thrm:sparse_saddle_point}), a sparse mini-max theorem (Theorem~\ref{thrm:sparse_mini_max}) and a sparse strong duality theorem (Theorem~\ref{thrm:sparse_strong_duality}).
  \item Dual optimization: we proposed an IHT-style algorithm along with its stochastic extension for non-smooth dual maximization. These algorithms have been shown to converge at the rate of $\frac{1}{\epsilon} \ln \frac{1}{\epsilon}$ in dual parameter estimation error and $\frac{1}{\epsilon^2} \ln \frac{1}{\epsilon^2}$ in primal-dual gap (see Theorem~\ref{thrm:DIHT_Conv} and Theorem~\ref{thrm:SDIHT_Conv}). These guarantees are invariant to RIP conditions which are required by virtually all the primal IHT-style methods without using relaxed sparsity levels.
\end{itemize}
\noindent\textbf{Notation.} Before continuing, we define some notations to be used. Let $x \in \mathbb{R}^d$ be a vector and $F$ be an index set. We use $\mathrm{H}_F(x)$ to denote the truncation operator that restricts $x$ to the set $F$. $\mathrm{H}_k(x)$ is a truncation operator which preserves the top $k$ (in magnitude) entries of $x$ and sets the remaining to be zero. The notation $\supp(x)$ represents the index set of nonzero entries of $x$. We conventionally define $\|x\|_\infty = \max_{i}|[x]_i|$ and define $x_{\min}=\min_{i \in \supp(x)} |[x]_i|$. For a matrix $A$, $\sigma_{\max}(A)$ ($\sigma_{\min}(A)$) denotes its largest (smallest) singular value.

\noindent\textbf{Organization.} The rest of this paper is organized as follows: In \S\ref{sect:relatedwork} we briefly review some relevant work. In \S\ref{sect:theory} we develop a Lagrangian duality theory for sparsity-constrained minimization problems. The dual IHT-style algorithms along with convergence analysis are presented in \S\ref{sect:algorithm}. The numerical evaluation results are reported in \S\ref{sect:experiment}. Finally, the concluding remarks are made in \S\ref{sect:conclusion}. All the technical proofs are deferred to the appendix sections.

\section{Related Work}
\label{sect:relatedwork}
For generic convex objective beyond quadratic loss, the rate of convergence and parameter estimation error of IHT-style methods were established under proper RIP (or restricted strong condition number) bound conditions~\citep{Blumensath-TIT-2013,Yuan-ICML-2014,yuan2016exact}. In~\citep{jain2014iterative}, several relaxed variants of IHT-style algorithms were presented for which the estimation consistency can be established without requiring the RIP conditions. In~\citep{bahmani2013greedy}, a gradient support pursuit algorithm is proposed and analyzed. In large-scale settings where a full gradient evaluation on all data becomes a bottleneck, stochastic and variance reduction techniques have been adopted to improve the computational efficiency of IHT~\citep{nguyen2014linear,li2016stochastic,chen2016accelerated}.

Dual optimization algorithms have been widely used in various learning tasks including SVMs~\citep{hsieh2008dual} and multi-task learning~\citep{lapin2014scalable}. In recent years, stochastic dual optimization methods have gained significant attention in large-scale machine learning~\citep{shalev2013accelerated,shalev2013stochastic}. To further improve computational efficiency, some primal-dual methods are developed to alternately minimize the primal objective and maximize the dual objective. The successful examples of primal-dual methods include learning total variation regularized model~\citep{chambolle2011first} and generalized Dantzig selector~\citep{lee2015fast}. More recently, a number of stochastic variants~\citep{zhang2015stochastic,yu2015doubly} and parallel variants~\citep{zhu2015stochastic} were developed to make the primal-dual algorithms more scalable and efficient.

\section{A Sparse Lagrangian Duality Theory}
\label{sect:theory}
In this section, we establish weak and strong duality theory that guarantees the original non-convex and NP-hard problem in~\eqref{prob:general} can be equivalently solved in a dual space. The results in this part build the theoretical foundation of developing dual IHT methods.

From now on we abbreviate $l_i(w^\top x_i) = l(w^\top x_i,y_i)$. The convexity of $l(w^\top x_i,y_i)$ implies that $l_i(u)$ is also convex. Let $l^*_i(\alpha_i)=\max\limits_{u}\{\alpha_iu - l_i(u)\}$ be the convex conjugate of $l_i(u)$ and $\mathcal{F}\subseteq \mathbb{R}$ be the feasible set of $\alpha_i$. According to the well-known expression of $l_i(u)=\max_{\alpha_i \in \mathcal{F}} \left\{\alpha_i u - l^*_i(\alpha_i)\right\}$, the problem~\eqref{prob:general} can be reformulated into the following mini-max formulation:
\begin{equation}\label{eqn:primaldual}
\min_{\|w\|_0 \le k} \frac{1}{N}\sum\limits_{i=1}^N\max_{\alpha_i \in \mathcal{F}}\{\alpha_i w^\top x_i - l^*_i(\alpha_i)\} + \frac{\lambda}{2}\|w\|^2.
\end{equation}
The following Lagrangian form will be useful in analysis:
\[
L(w,\alpha) = \frac{1}{N}\sum\limits_{i=1}^N\left(\alpha_iw^\top x_i - l^*_i(\alpha_i)\right) + \frac{\lambda}{2}\|w\|^2,
\]
where $\alpha=[\alpha_1,...,\alpha_N] \in \mathcal{F}^N$ is the vector of dual variables. We now introduce the following concept of sparse saddle point which is a restriction of the conventional saddle point to the setting of sparse optimization.
\begin{definition}[Sparse Saddle Point]
A pair $(\bar w, \bar \alpha)\in \mathbb{R}^d \times \mathcal{F}^N$ is said to be a $k$-sparse saddle point for $L$ if $\|\bar w\|_0 \le k$ and the following holds for all $\|w\|_0 \le k, \alpha \in \mathcal{F}^N$:
\begin{equation}\label{inequat:sparse_saddle_point}
L(\bar w, \alpha) \le L(\bar w, \bar \alpha) \le L(w, \bar \alpha).
\end{equation}
\end{definition}
Different from the conventional definition of saddle point, the $k$-sparse saddle point only requires the inequality~\eqref{inequat:sparse_saddle_point} holds for any arbitrary $k$-sparse vector $w$. The following result is a basic sparse saddle point theorem for $L$. Throughout the paper, we will use $f'(\cdot)$ to denote a sub-gradient (or super-gradient) of a convex (or concave) function $f(\cdot)$, and use $\partial f(\cdot)$ to denote its sub-differential (or super-differential).
\begin{theorem}[Sparse Saddle Point Theorem]\label{thrm:sparse_saddle_point}
Let $\bar w \in \mathbb{R}^d$ be a $k$-sparse primal vector and $\bar \alpha \in \mathcal{F}^N$ be a dual vector. Then the pair $(\bar w, \bar \alpha)$ is a sparse saddle point for $L$ if and only if the following conditions hold:
\begin{itemize}
  \item[(a)] $\bar w$ solves the primal problem in~\eqref{prob:general};
  \item[(b)] $\bar \alpha \in [\partial l_1(\bar w^\top x_1),...,\partial l_N(\bar w^\top x_N)]$;
  \item[(c)] $\bar w = \mathrm{H}_k\left(-\frac{1}{\lambda N}\sum_{i=1}^N \bar\alpha_i x_i\right)$.
\end{itemize}
\end{theorem}
\begin{proof}
A proof of this result is given in Appendix~\ref{append:proof_sparse_saddle_point}.
\end{proof}
\begin{remark}
Theorem~\ref{thrm:sparse_saddle_point} shows that the conditions~(a)$\sim$(c) are sufficient and necessary to guarantee the existence of a sparse saddle point for the Lagrangian form $L$. This result is different from from the traditional saddle point theorem which requires the use of the Slater Constraint Qualification to guarantee the existence of saddle point.
\end{remark}
\begin{remark}
Let us consider $P'(\bar w)=\frac{1}{N}\sum_{i=1}^N \bar\alpha_i x_i + \lambda \bar w \in \partial P(\bar w)$. Denote $\bar F = \supp(\bar w)$. It is easy to verify that the condition~(c) in Theorem~\ref{thrm:sparse_saddle_point} is equivalent to
\[
\mathrm{H}_{\bar F}(P'(\bar w)) =0, \ \ \ \bar w_{\min} \ge \frac{1}{\lambda}\|P'(\bar w)\|_\infty.
\]
\end{remark}

The following sparse mini-max theorem guarantees that the min and max in~\eqref{eqn:primaldual} can be safely switched if and only if there exists a sparse saddle point for $L(w,\alpha)$.
\begin{theorem}[Sparse Mini-Max Theorem]\label{thrm:sparse_mini_max}
The mini-max relationship
\begin{equation}\label{equat:mini-max}
\max_{\alpha \in \mathcal{F}^N} \min_{\|w\|_0\le k} L(w, \alpha) =\min_{\|w\|_0\le k}\max_{\alpha \in \mathcal{F}^N} L(w,\alpha)
\end{equation}
holds if and only if there exists a sparse saddle point $(\bar w, \bar \alpha)$ for $L$.
\end{theorem}
\begin{proof}
A proof of this result is given in Appendix~\ref{append:proof_sparse_mini_max}.
\end{proof}
The sparse mini-max result in Theorem~\ref{thrm:sparse_mini_max} provides sufficient and necessary conditions under which one can safely exchange a min-max for a max-min, in the presence of sparsity constraint. The following corollary is a direct consequence of applying Theorem~\ref{thrm:sparse_saddle_point} to Theorem~\ref{thrm:sparse_mini_max}.
\begin{corollary}\label{corol:sparse_mini_max}
The mini-max relationship
\[
\max_{\alpha \in \mathcal{F}^N} \min_{\|w\|_0\le k} L(w, \alpha) =\min_{\|w\|_0\le k}\max_{\alpha \in \mathcal{F}^N} L(w,\alpha)
\]
holds if and only if there exist a $k$-sparse primal vector $\bar w \in \mathbb{R}^d$ and a dual vector $\bar \alpha \in \mathcal{F}^N$ such that the conditions~(a)$\sim$(c) in Theorem~\ref{thrm:sparse_saddle_point} are satisfied.
\end{corollary}

The mini-max result in Theorem~\ref{thrm:sparse_mini_max} can be used as a basis for establishing sparse duality theory. Indeed, we have already shown the following:
\[
\min_{\|w\|_0 \le k}\max_{\alpha \in \mathcal{F}^N} L(w,\alpha) = \min_{\|w\|_0\le k} P(w).
\]
This is called the \emph{primal} minimization problem and it is the min-max side of the sparse mini-max theorem. The other side, the max-min problem, will be called as the \emph{dual} maximization problem with dual objective function $D(\alpha):=\min_{\|w\|_0 \le k} L(w, \alpha)$, i.e.,
\begin{equation}\label{prob:dual}
\max_{\alpha \in \mathcal{F}^N}D(\alpha)=\max_{\alpha \in \mathcal{F}^N}\min_{\|w\|_0 \le k} L(w, \alpha).
\end{equation}
The following Lemma~\ref{lemma:concave} shows that the dual objective function $D(\alpha)$ is concave and explicitly gives the expression of its super-differential.
\begin{lemma}\label{lemma:concave}
The dual objective function $D(\alpha)$ is given by
\[
D(\alpha)= \frac{1}{N}\sum_{i=1}^N - l^*_i(\alpha_i) - \frac{\lambda}{2}\|w(\alpha)\|^2,
\]
where $ w(\alpha) = \mathrm{H}_k\left(-\frac{1}{N\lambda} \sum_{i=1}^N \alpha_i x_i \right)$.
Moreover, $D(\alpha)$ is concave and its super-differential is given by
\[
\partial D(\alpha) = \frac{1}{N} [w(\alpha)^\top x_1 - \partial l^*_1(\alpha_1),...,w(\alpha)^\top x_N - \partial l^*_N(\alpha_N)].
\]
Particularly, if $w(\alpha)$ is unique at $\alpha$ and $\{l^*_i\}_{i=1,...,N}$ are differentiable, then $\partial D(\alpha)$ is unique and it is the super-gradient of $D(\alpha)$.
\end{lemma}
\begin{proof}
A proof of this result is given in Appendix~\ref{append:proof_concave}.
\end{proof}
Based on Theorem~\ref{thrm:sparse_saddle_point} and Theorem~\ref{thrm:sparse_mini_max}, we are able to further establish a sparse strong duality theorem which gives the sufficient and necessary conditions under which the optimal values of the primal and dual problems coincide.
\begin{theorem}[Sparse Strong Duality Theorem]\label{thrm:sparse_strong_duality}
Let $\bar w \in \mathbb{R}^d$ is a $k$-sparse primal vector and $\bar \alpha \in \mathcal{F}^N$ be a dual vector. Then $\bar \alpha$ solves the dual problem in~\eqref{prob:dual}, i.e., $D(\bar\alpha) \ge D(\alpha), \ \forall \alpha \in \mathcal{F}^N$, and $P(\bar w) = D(\bar\alpha)$ if and only if the pair $(\bar w, \bar\alpha)$ satisfies the conditions~(a)$\sim$(c) in Theorem~\ref{thrm:sparse_saddle_point}.
\end{theorem}
\begin{proof}
A proof of this result is given in Appendix~\ref{append:proof_sparse_strong_duality}.
\end{proof}
We define the sparse primal-dual gap $\epsilon_{PD}(w,\alpha):=P(w) - D(\alpha)$. The main message conveyed by Theorem~\ref{thrm:sparse_strong_duality} is that the sparse primal-dual gap reaches zero at the primal-dual pair $(\bar w, \bar \alpha)$ if and only if the conditions~(a)$\sim$(c) in Theorem~\ref{thrm:sparse_saddle_point} hold.

The sparse duality theory developed in this section suggests a natural way for finding the global minimum of the sparsity-constrained minimization problem in~\eqref{prob:general} via maximizing its dual problem in~\eqref{prob:dual}. Once the dual maximizer $\bar\alpha$ is estimated, the primal sparse minimizer $\bar w$ can then be recovered from it according to the prima-dual connection $\bar w = \mathrm{H}_k\left(-\frac{1}{\lambda N}\sum_{i=1}^N \bar\alpha_i x_i\right)$ as given in the condition~(c). Since the dual objective function $D(\alpha)$ is shown to be concave, its global maximum can be estimated using any convex/concave optimization method. In the next section, we present a simple projected super-gradient method to solve the dual maximization problem.

\section{Dual Iterative Hard Thresholding}
\label{sect:algorithm}
Generally, $D(\alpha)$ is a non-smooth function since: 1) the conjugate function $l^*_i$ of an arbitrary convex loss $l_i$ is generally non-smooth and 2) the term $\|w(\alpha)\|^2$ is non-smooth with respect to $\alpha$ due to the truncation operation involved in computing $w(\alpha)$. Therefore, smooth optimization methods are not directly applicable here and we resort to sub-gradient-type methods to solve the non-smooth dual maximization problem in~\eqref{prob:dual}.

\subsection{Algorithm}
The Dual Iterative Hard Thresholding (DIHT) algorithm, as outlined in Algorithm~\ref{alg:DIHT}, is essentially a projected super-gradient method for maximizing $D(\alpha)$. The procedure generates a sequence of prima-dual pairs $(w^{(0)},\alpha^{(0)}), (w^{(1)},\alpha^{(1)}), \ldots$ from an
initial pair $w^{(0)}=0$ and $\alpha^{(0)}=0$. At the $t$-th iteration, the dual update step \textbf{S1} conducts the projected super-gradient ascent in~\eqref{eqn:dualupdate1} to update $\alpha^{(t)}$ from $\alpha^{(t-1)}$ and $w^{(t-1)}$. Then in the primal update step \textbf{S2}, the primal variable $w^{(t)}$ is constructed from $\alpha^{(t)}$ using a $k$-sparse truncation operation in~\eqref{eqn:primalupdate}.

\begin{algorithm}[tb]\caption{Dual Iterative Hard Thresholding (DIHT)}
\label{alg:DIHT}
\SetKwInOut{Input}{Input}\SetKwInOut{Output}{Output}\SetKw{Initialization}{Initialization}
\Input{Training set $\{x_i,y_i\}_{i=1}^N$. Regularization strength parameter $\lambda$. Cardinality constraint $k$. Step-size $\eta$. }
\Initialization{$w^{(0)}=0$, $\alpha_1^{(0)}=...=\alpha_N^{(0)}=0$.} \\
\For{ $t=1,2,...,T$}{
(\textbf{S1}) Dual projected super-gradient ascent: $\forall~ i\in \{1,2,...,N\}$,
\begin{equation}\label{eqn:dualupdate1}
\alpha_i^{(t)} = \mathrm{P}_{\mathcal{F}}\left(\alpha_i^{(t-1)}+\eta^{(t-1)} g_i^{(t-1)}\right),
\end{equation}
where $g_i^{(t-1)}=\frac{1}{N}(x_i^\top w^{(t-1)}- l^{*'}_i(\alpha_i^{(t-1)}))$ is the super-gradient and $\mathrm{P}_{\mathcal{F}}(\cdot)$ is the Euclidian projection operator with respect to feasible set $\mathcal{F}$.

(\textbf{S2}) Primal hard thresholding:\\
\begin{equation}\label{eqn:primalupdate}
w^{(t)}= \mathrm{H}_k\left(-\frac{1}{\lambda N}\sum\limits_{i=1}^N \alpha_i^{(t)}x_i\right).
\end{equation}
}
\Output{$w^{(T)}$.}
\end{algorithm}

When a batch estimation of super-gradient $D'(\alpha)$ becomes expensive in large-scale applications, it is natural to consider the stochastic implementation of DIHT, namely SDIHT, as outlined in Algorithm~\ref{alg:SDIHT}. Different from the batch computation in Algorithm~\ref{alg:DIHT}, the dual update step \textbf{S1} in Algorithm~\ref{alg:SDIHT} randomly selects a block of samples (from a given block partition of samples) and update their corresponding dual variables according to~\eqref{eqn:stodualupdate}. Then in the primal update step \textbf{S2.1}, we incrementally update an intermediate accumulation vector $\tilde w^{(t)}$ which records $-\frac{1}{\lambda N}\sum_{i=1}^N \alpha_i^{(t)}x_i$ as a weighted sum of samples. In \textbf{S2.2}, the primal vector $w^{(t)}$ is updated by applying $k$-sparse truncation on $\tilde w^{(t)}$. The SDIHT is essentially a block-coordinate super-gradient method for the dual problem. Particularly, in the extreme case $m=1$, SDIHT reduces to the batch DIHT. At the opposite extreme end with $m=N$, i.e., each block contains one sample, SDIHT becomes a stochastic coordinate-wise super-gradient method.

The dual update~\eqref{eqn:stodualupdate} in SDIHT is much more efficient than DIHT as the former only needs to access a small subset of samples at a time. If the hard thresholding operation in primal update becomes a bottleneck, e.g., in high-dimensional settings, we suggest to use SDIHT with relatively smaller number of blocks so that the hard thresholding operation in \textbf{S2.2} can be less frequently called.

\begin{algorithm}[tb]\caption{Stochastic Dual Iterative Hard Thresholding (SDIHT)}
\label{alg:SDIHT}
\SetKwInOut{Input}{Input}\SetKwInOut{Output}{Output}\SetKw{Initialization}{Initialization}
\Input{Training set $\{x_i,y_i\}_{i=1}^N$. Regularization strength parameter $\lambda$. Cardinality constraint $k$. Step-size $\eta$. A block disjoint partition $\{B_1,...,B_m\}$ of the sample index set $[1,...,N]$.}
\Initialization{$w^{(0)}=\tilde w^{(0)} = 0$, $\alpha_1^{(0)}=...=\alpha_N^{(0)}=0$.} \\
\For{ $t=1,2,...,T$}{
(S1) Dual projected super-gradient ascent: Uniformly randomly select an index block $B^{(t)}_i \in \{B_1,...,B_m\}$. For all $j\in B^{(t)}_i$ update $\alpha^{(t)}_j$ as
\begin{equation}\label{eqn:stodualupdate}
\alpha_j^{(t)} = \mathrm{P}_{\mathcal{F}}\left(\alpha_j^{(t-1)}+\eta^{(t-1)} g_j^{(t-1)}\right).
\end{equation}
Set $\alpha_j^{(t)}=\alpha_j^{(t-1)}$, $\forall j \notin B^{(t)}_i$.

(S2) Primal hard thresholding:

\ \ \ -- (S2.1) Intermediate update:
 \begin{equation}\label{eqn:intprimalupdate}
\tilde{w}^{(t)}=\tilde w^{(t-1)}-\frac{1}{\lambda N}\sum_{j\in B^{(t)}_i}(\alpha_j^{(t)}-\alpha_j^{(t-1)})x_j.
\end{equation}

\ \ \ -- (S2.2) Hard thresholding: $w^{(t)}= \mathrm{H}_k(\tilde{w}^{(t)})$.
}
\Output{$w^{(T)}$.}
\end{algorithm}

\subsection{Convergence analysis}

We now analyze the non-asymptotic convergence behavior of DIHT and SDIHT. In the following analysis, we will denote $\bar\alpha = \argmax_{\alpha \in \mathcal{F}^N} D(\alpha)$ and use the abbreviation $\epsilon_{PD}^{(t)}:= \epsilon_{PD}(w^{(t)},\alpha^{(t)})$. Let $r=\max_{a\in \mathcal{F}}|a|$ be the bound of the dual feasible set $\mathcal{F}$ and $\rho=\max_{i, a\in \mathcal{F}}| l_i^{*'}(a)|$. For example, such quantities exist when $l_i$ and $l^*_i$ are Lipschitz continuous~\citep{shalev2013stochastic}. We also assume without loss of generality that $\|x_i\|\le 1$. Let $X=[x_1,...,x_N]\in \mathbb{R}^{d\times N}$ be the data matrix. Given an index set $F$, we denote $X_F$ as the restriction of $X$ with \emph{rows} restricted to $F$. The following quantities will be used in our analysis:
\[
\begin{aligned}
\sigma_{\max}^2 (X,s) &= \sup_{u\in \mathbb{R}^n, F} \left\{ u^\top X_F^\top X_F u \mid |F| \le s, \|u\|=1 \right\}, \\
\sigma_{\min}^2(X,s)  &= \inf_{u\in \mathbb{R}^n, F} \left\{ u^\top X_F^\top X_F u \mid |F| \le s, \|u\| = 1 \right\}.
\end{aligned}
\]
Particularly, $\sigma_{\max}(X,d) = \sigma_{\max}(X)$ and $\sigma_{\min}(X,d) = \sigma_{\min}(X)$. We say a univariate differentiable function $f(x)$ is $\gamma$-smooth if $\forall x, y$, $f(y) \le f(x) + \langle f'(x), y - x \rangle + \frac{\gamma}{2} |x - y|^2$. The following is our main theorem on the dual parameter estimation error, support recovery and primal-dual gap of DIHT.
\begin{theorem}\label{thrm:DIHT_Conv}
Assume that $l_i$ is $1/\mu$-smooth. Set $\eta^{(t)} = \frac{\lambda N^2}{(\lambda N \mu + \sigma_{\min}(X,k))(t+1)}$. Define constants $c_1 = \frac{N^3 (r+\lambda\rho)^2}{(\lambda N \mu + \sigma_{\min}(X,k))^2}$ and $c_2=(r+\lambda\rho)^2\left(1+\frac{\sigma_{\max}(X,k)}{\mu\lambda N}\right)^2$.
\begin{itemize}
  \item[(a)] \emph{\textbf{Parameter estimation error}:} The sequence $\{\alpha^{(t)}\}_{t\ge 1}$ generated by Algorithm~\ref{alg:DIHT} satisfies the following estimation error inequality:
\[
\|\alpha^{(t)} - \bar\alpha\|^2  \le c_1\left(\frac{1}{t}  + \frac{\ln t}{t}\right),
\]
\item[(b)]\emph{\textbf{Support recovery and primal-dual gap}:} Assume additionally that $\bar\epsilon: = \bar w_{\min} - \frac{1}{\lambda} \|P'(\bar w)\|_\infty >0$. Then, $\supp(w^{(t)})= \supp(\bar w)$ when
  \[
t \ge t_0= \left\lceil\frac{12c_1\sigma_{\max}^2(X)}{\lambda^2 N^2 \bar\epsilon^2} \ln \frac{12c_1\sigma_{\max}^2(X)}{\lambda^2 N^2 \bar\epsilon^2}\right\rceil.
\]
Moreover, for any $\epsilon>0$, the primal-dual gap satisfies $\epsilon_{PD}^{(t)} \le \epsilon$ when $t \ge \max\{t_0, t_1\}$ where $t_1=\left\lceil\frac{3c_1c_2}{\lambda^2 N\epsilon^2}\ln\frac{3c_1c_2}{\lambda^2 N \epsilon^2}\right\rceil$.
\end{itemize}
\end{theorem}
\begin{proof}
A proof of this result is given in Appendix~\ref{append:proof_DIHT_conv}.
\end{proof}
\begin{remark}
The theorem allows $\mu = 0$ when $\sigma_{\min}(X,k)>0$. If $\mu >0$, then $\sigma_{\min}(X,k)$ is allowed to be zero and thus the step-size can be set as $\eta^{(t)}=\frac{N}{\mu(t+1)}$.
\end{remark}
Consider primal sub-optimality $\epsilon_P^{(t)}:=P(w^{(t)}) - P(\bar w)$. Since $\epsilon_P^{(t)} \le \epsilon_{PD}^{(t)}$ always holds, the convergence rates in Theorem~\ref{thrm:DIHT_Conv} are applicable to the primal sub-optimality as well. An interesting observation is that these convergence results on $\epsilon_P^{(t)}$ are not relying on the Restricted Isometry Property (RIP) (or restricted strong condition number) which is required in most existing analysis of IHT-style algorithms~\citep{blumensath2009iterative,Yuan-ICML-2014}. In \citep{jain2014iterative}, several relaxed variants of IHT-style algorithms are presented for which the estimation consistency can be established without requiring the RIP conditions. In contrast to the RIP-free sparse recovery analysis in~\citep{jain2014iterative}, our Theorem~\ref{thrm:DIHT_Conv} does not require the sparsity level $k$ to be relaxed.

For SDIHT, we can establish similar non-asymptotic convergence results as summarized in the following theorem.
\begin{theorem}\label{thrm:SDIHT_Conv}
Assume that $l_i$ is $1/\mu$-smooth. Set $\eta^{(t)} = \frac{\lambda m N^2}{(\lambda N \mu + \sigma_{\min}(X,k))(t+1)}$.
\begin{itemize}
  \item[(a)] \emph{\textbf{Parameter estimation error}:} The sequence $\{\alpha^{(t)}\}_{t\ge 1}$ generated by Algorithm~\ref{alg:SDIHT} satisfies the following expected estimation error inequality:
\[
\mathbb{E}[\|\alpha^{(t)} - \bar\alpha\|^2] \le m c_1\left(\frac{1}{t}  + \frac{\ln t}{t}\right),
\]
\item[(b)] \emph{\textbf{Support recovery and primal-dual gap}:} Assume additionally that $\bar\epsilon: = \bar w_{\min} - \frac{1}{\lambda} \|P'(\bar w)\|_\infty >0$. Then, for any $\delta \in (0,1)$, with probability at least $1-\delta$, it holds that $\supp(w^{(t)})= \supp(\bar w)$ when
  \[
t \ge t_2= \left\lceil\frac{12mc_1\sigma_{\max}^2(X)}{\lambda^2\delta^2 N^2 \bar\epsilon^2} \ln \frac{12mc_1\sigma_{\max}^2(X)}{\lambda^2\delta^2 N^2 \bar\epsilon^2}\right\rceil.
\]
Moreover, with probability at least $1-\delta$, the primal-dual gap satisfies $\epsilon_{PD}^{(t)} \le \epsilon$ when $t \ge \max\{4t_2, t_3\}$ where $ t_3= \left\lceil\frac{12mc_1c_2}{\lambda^2\delta^2 N\epsilon^2} \ln \frac{12mc_1c_2}{\lambda^2\delta^2 N \epsilon^2}\right\rceil$.
\end{itemize}
\end{theorem}
\begin{proof}
A proof of this result is given in Appendix~\ref{append:proof_SDIHT_conv}.
\end{proof}
\begin{remark}
Theorem~\ref{thrm:SDIHT_Conv} shows that, up to scaling factors, the expected or high probability iteration complexity of SDIHT is almost identical to that of DIHT. The scaling factor $m$ appeared in $t_2$ and $t_3$ reflects a trade-off between the decreased per-iteration cost and the increased iteration complexity.
\end{remark}

\section{Experiments}

This section dedicates in demonstrating the accuracy and efficiency of the proposed algorithms. We first show the model estimation performance of DIHT when applied to sparse ridge regression models on synthetic datasets. Then we evaluate the efficiency of DIHT/SDIHT on sparse $\ell_2$-regularized Huber loss and Hinge loss minimization tasks using real-world datasets.
\subsection{Model parameter estimation accuracy evaluation}
\label{sect:experiment}
\begin{figure}
\centering
\subfigure[Model estimation error]{
\includegraphics[width=0.22\textwidth,height=0.21\textwidth]{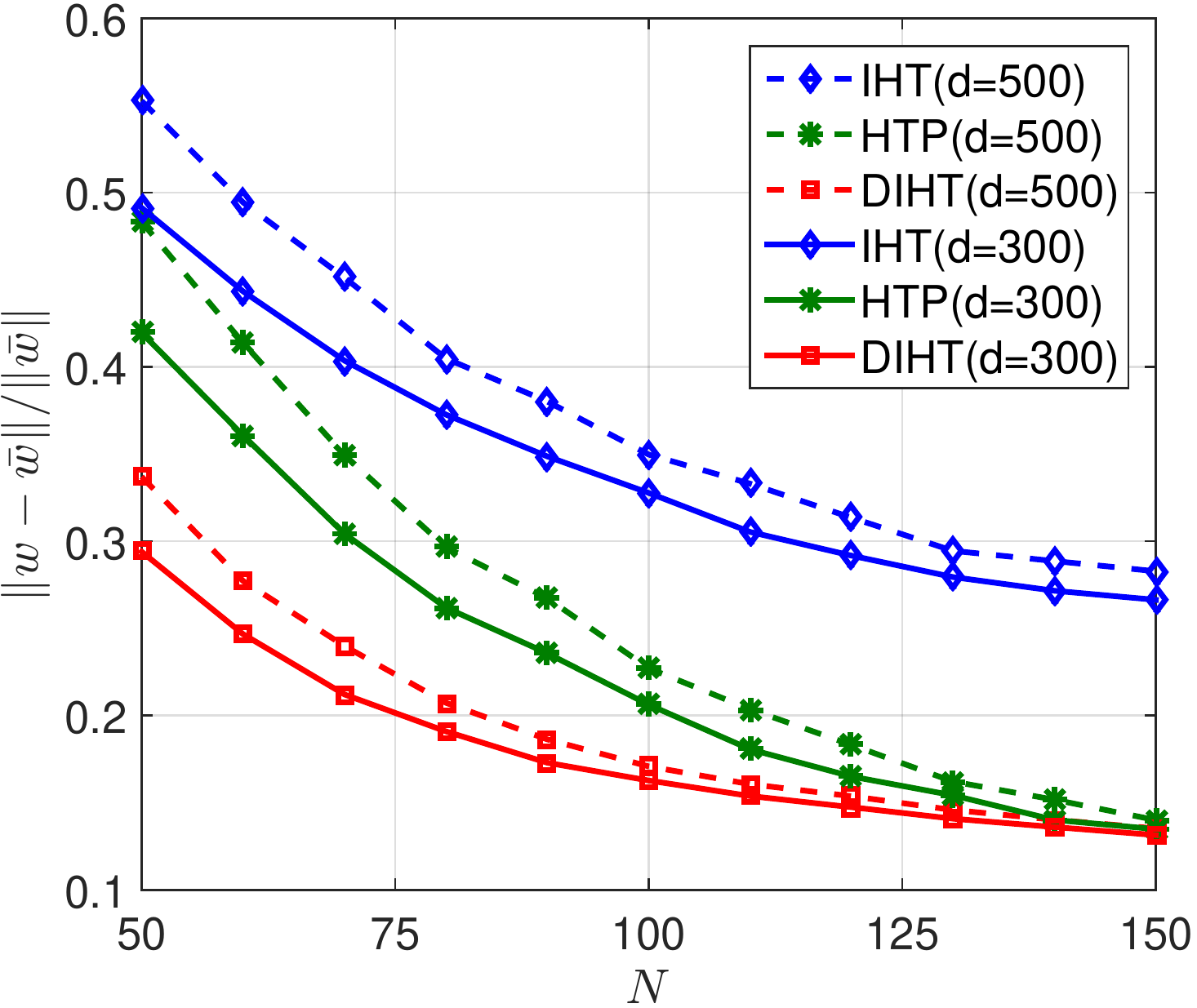}
\label{fig:modelesterror}
}
\subfigure[Percentage of support recovery success]{
\includegraphics[width=0.22\textwidth,height=0.21\textwidth]{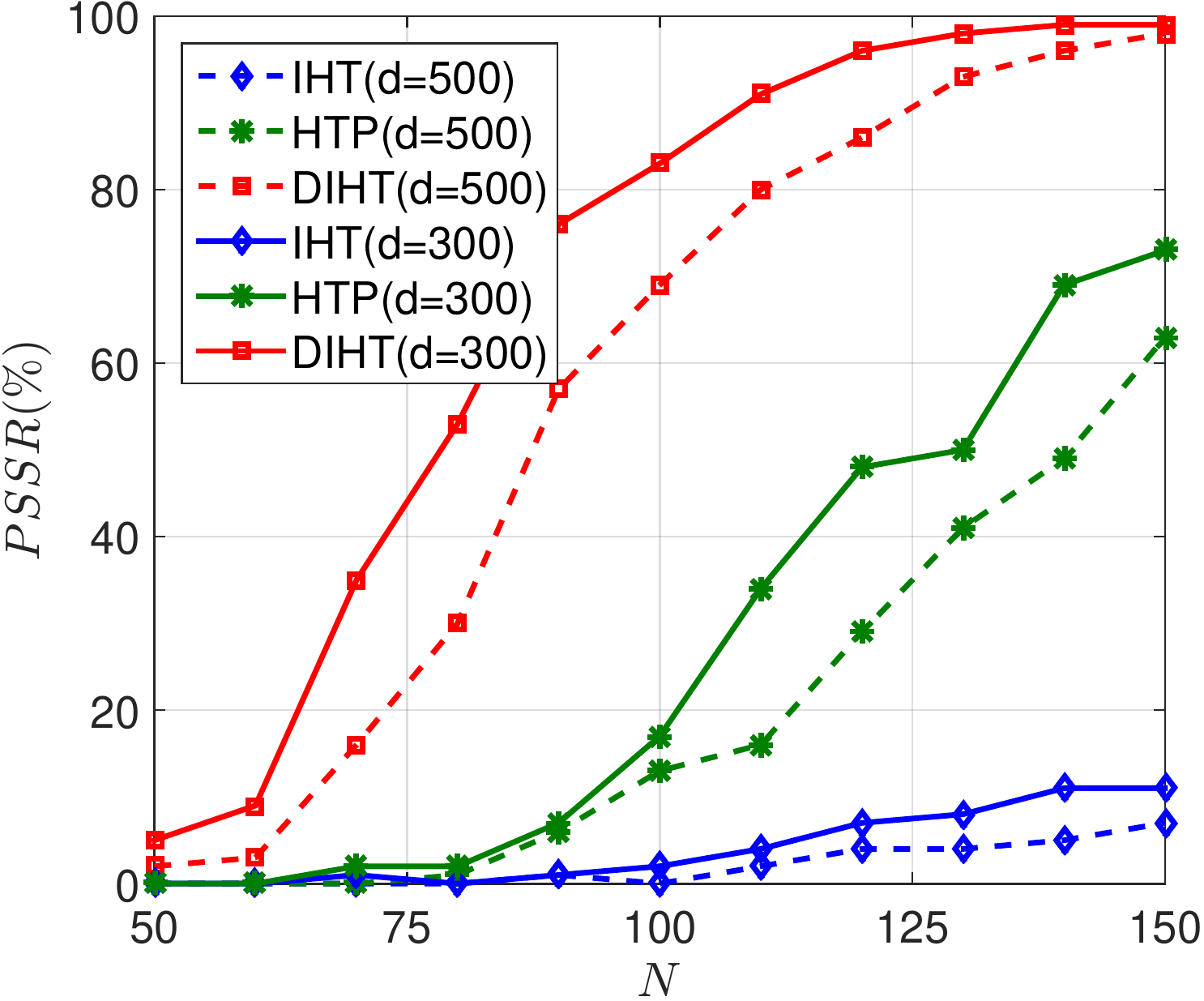}
\label{fig:precsuccrecov}
}
\caption{Model parameter estimation performance comparison between DIHT and baseline algorithms on the two synthetic dataset settings. The varying number of training sample is denoted by $N$.}
\label{fig:supportrecovery}
\end{figure}
A synthetic model is generated with sparse model parameter $\bar{w}=[\underbrace{1,1,\cdots,1}_{\bar{k}},\underbrace{0,0,\cdots,0}_{d-\bar{k}}]$. Each $x_i\in \mathbb{R}^{d}$ of the $N$ training data examples $\{x_i\}_{i=1}^N$ is designed to have two components. The first component is the top-$\bar{k}$ feature dimensions drawn from multivariate Gaussian distribution $N(\mu_1,\Sigma)$. Each entry in $\mu_1\in \mathbb{R}^{\bar{k}}$ independently follows standard normal distribution. The entries of covariance $\Sigma_{ij}=\left\{\begin{array}{ll} 1 & i=j\\
0.25 & i\ne j
\end{array}\right..$
The second component consists the left $d-\bar{k}$ feature dimensions. It follows $N(\mu_2, I)$ where each entry in $\mu_2\in \mathbb{R}^{d-\bar{k}}$ is drawn from standard normal distribution. We simulate two data parameter settings: (1) $d=500,\bar{k}=100$; (2) $d=300, \bar{k}=100$. In each data parameter setting 150 random data copies are produced independently. The task is to solve the following $\ell_2$-regularized sparse linear regression problem:
\vspace{-5.5mm}
\[
\min_{\|w\|\le k} \frac{1}{N}\sum\limits_{i=1}^Nl_{sq}(y_i,w^{\top}x_i)+\frac{\lambda}{2}\|w\|^2,
\]
where $l_{sq}(y_i,w^{\top}x_i)=(y_i-w^{\top}x_i)^2$. The responses $\{y_i\}_{i=1}^N$ are produced by $y_i=\bar{w}^{\top}x_i+\varepsilon_i$, where $\varepsilon_i\sim N(0,1)$. The convex conjugate of $l_{sq}(y_i,w^{\top}x_i)$ is known as $l_{sq}^*(\alpha_i)=\frac{\alpha_i^2}{4}+y_i\alpha_i$~\cite{shalev2013stochastic}. We consider solving the problem under the sparsity level $k=\bar{k}$. Two measurements are calculated for evaluation. The first is \emph{parameter estimation error} $\|w-\bar{w}\|/\|\bar{w}\|$. Apart from it we calculate the \emph{percentage of successful support recovery} ($PSSR$) as the second performance metric. A successful support recovery is obtained if $supp(\bar{w})= supp(w)$. The evaluation is conducted on the generated batch data copies to calculate the percentage of successful support recovery. We use 50 data copies as validation set to select the parameter $\lambda$ from $\{10^{-6},...,10^2\}$ and the percentage of successful support recovery is evaluated on the other 100 data copies.

Iterative hard thresholding (IHT)~\cite{blumensath2009iterative} and hard thresholding pursuit (HTP)~\cite{foucart2011hard} are used as the baseline primal algorithms. The parameter estimation error and percentage of successful support recovery curves under varying training size are illustrated in Figure~\ref{fig:supportrecovery}. We can observe from this group of curves that DIHT consistently achieves lower parameter estimation error and higher rate of successful support recovery than IHT and HTP. It is noteworthy that most significant performance gap between DIHT and the baselines occurs when the training size $N$ is comparable to or slightly smaller than the sparsity level $\bar{k}$.

\begin{figure*}
\centering
\subfigure[RCV1, $\lambda=0.002$]{
\includegraphics[width=0.21\textwidth,height=0.2\textwidth]{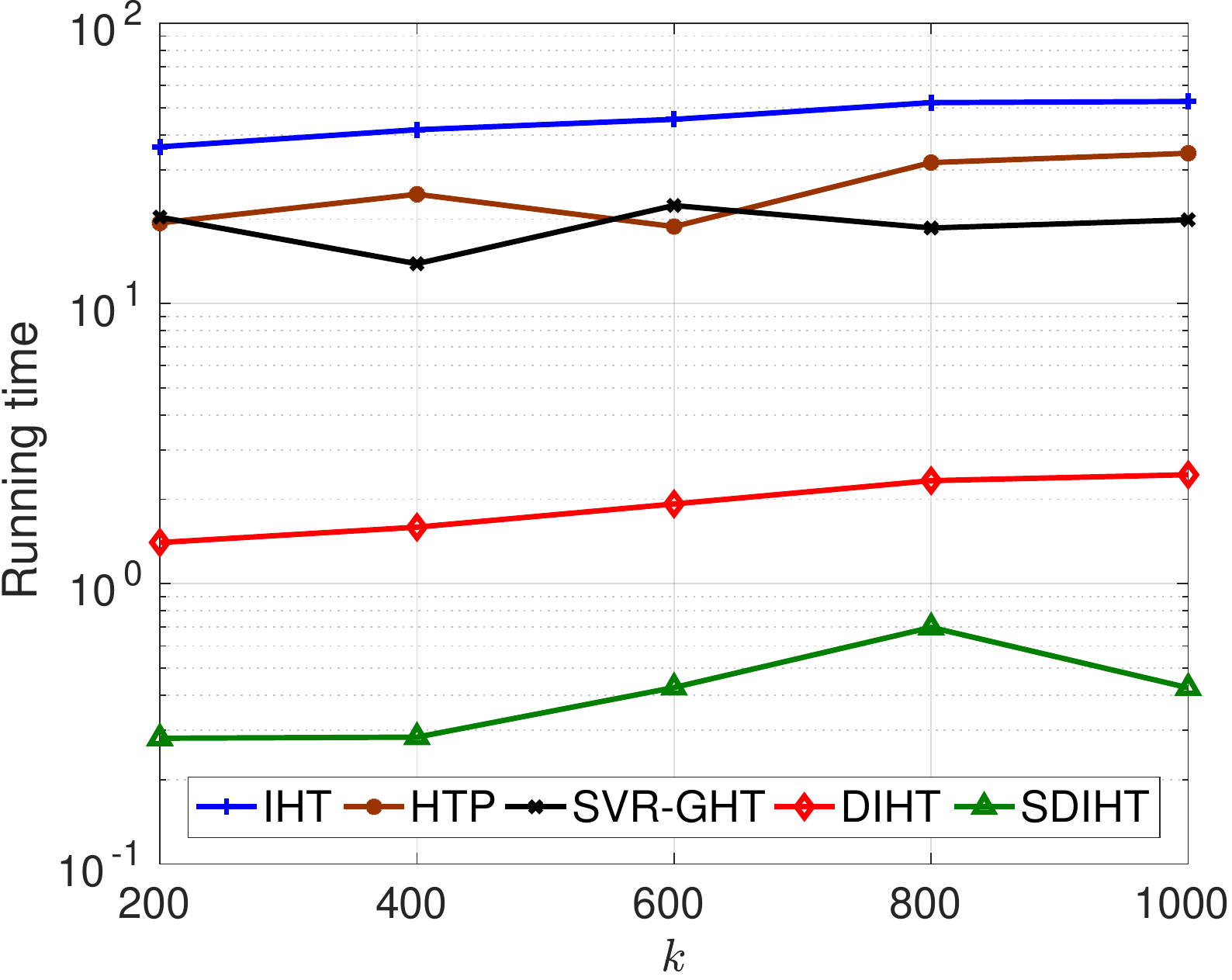}
\label{fig:TimeDIHTRCVSVM}
}
\subfigure[RCV1, $\lambda=0.0002$]{
\includegraphics[width=0.21\textwidth,height=0.2\textwidth]{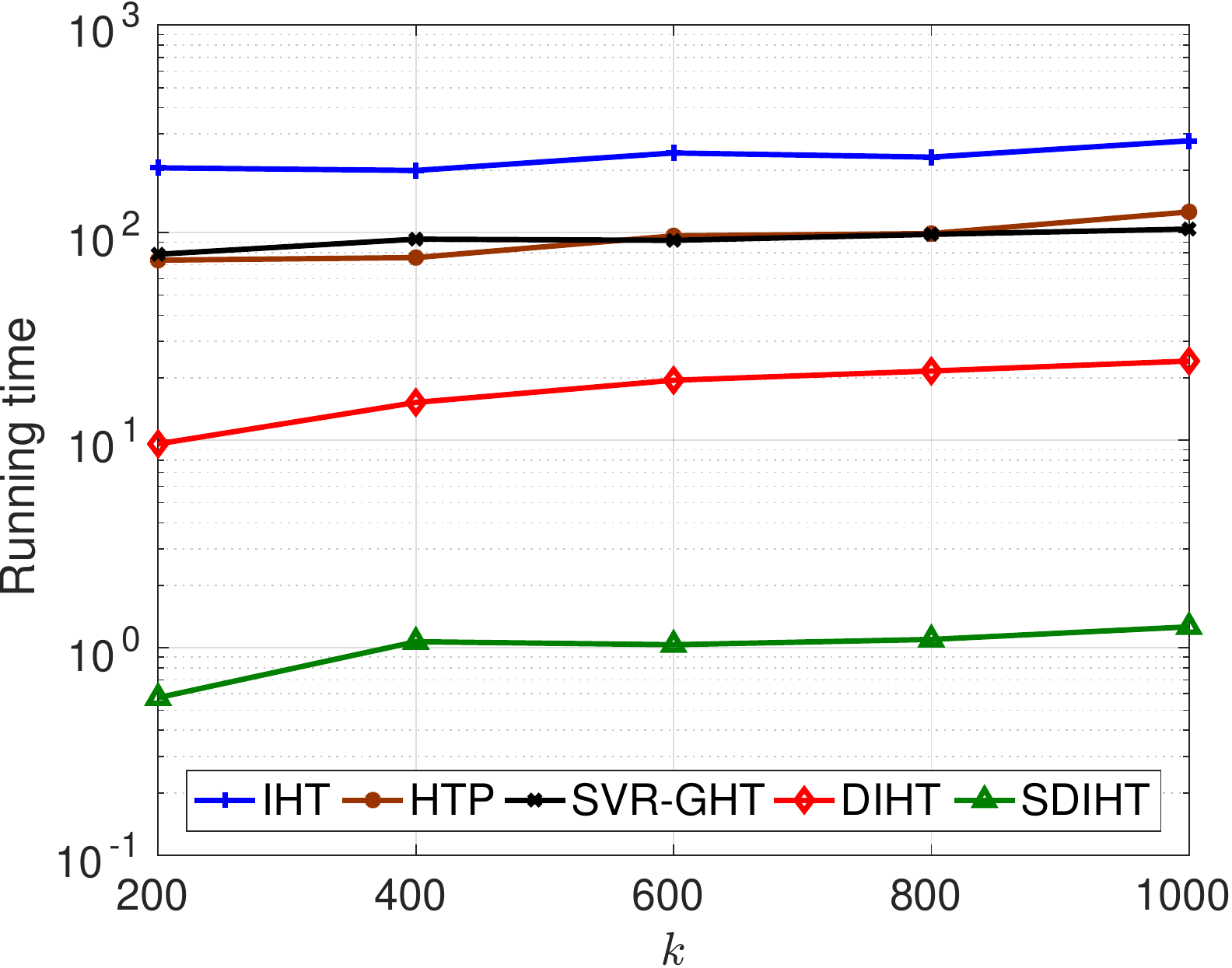}
\label{fig:TimeSDIHTRCVSVM}
}
\subfigure[News20, $\lambda=0.002$]{
\includegraphics[width=0.21\textwidth,height=0.2\textwidth]{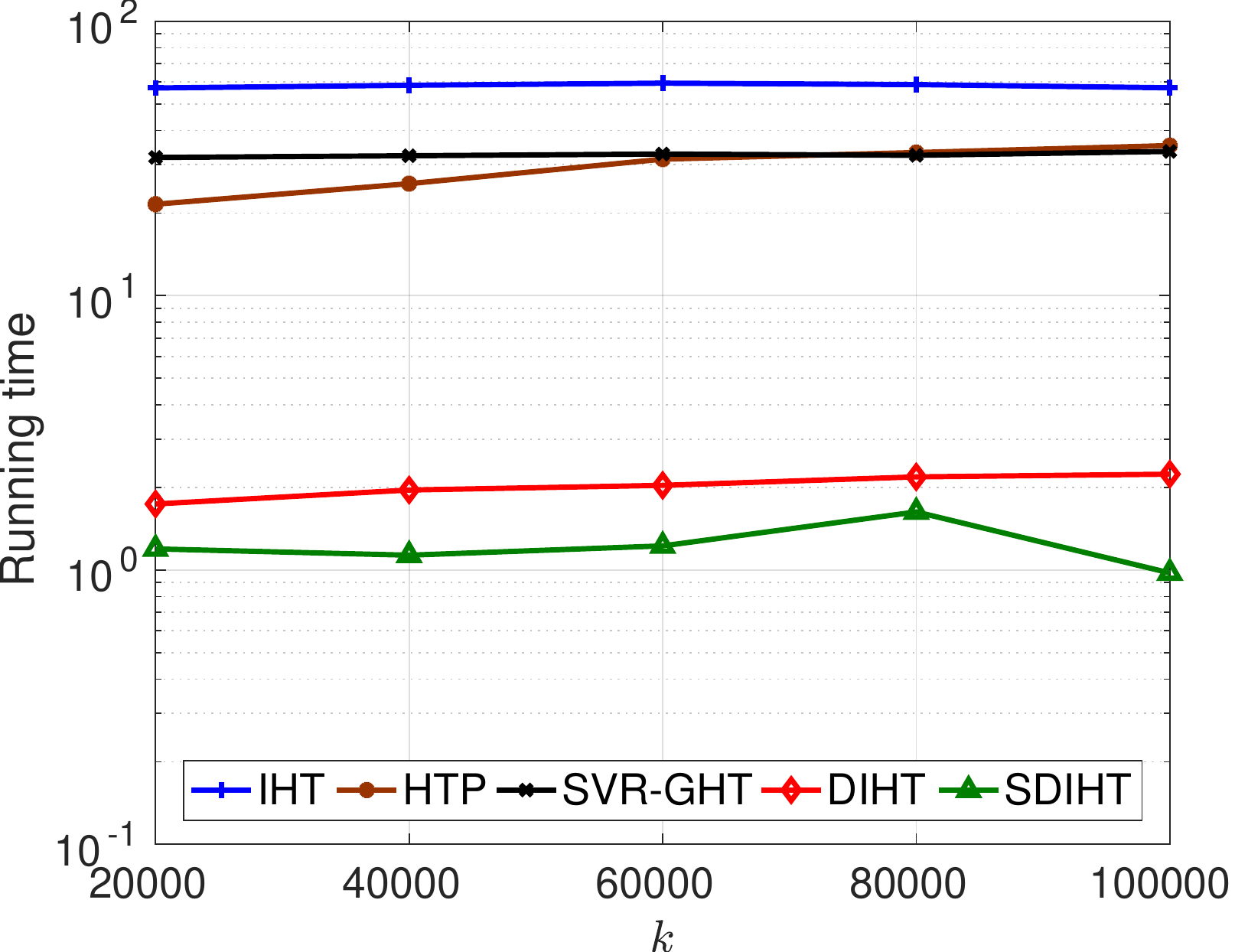}
\label{fig:TimeDIHTNewSVM}
}
\subfigure[News20, $\lambda=0.0002$]{
\includegraphics[width=0.21\textwidth,height=0.2\textwidth]{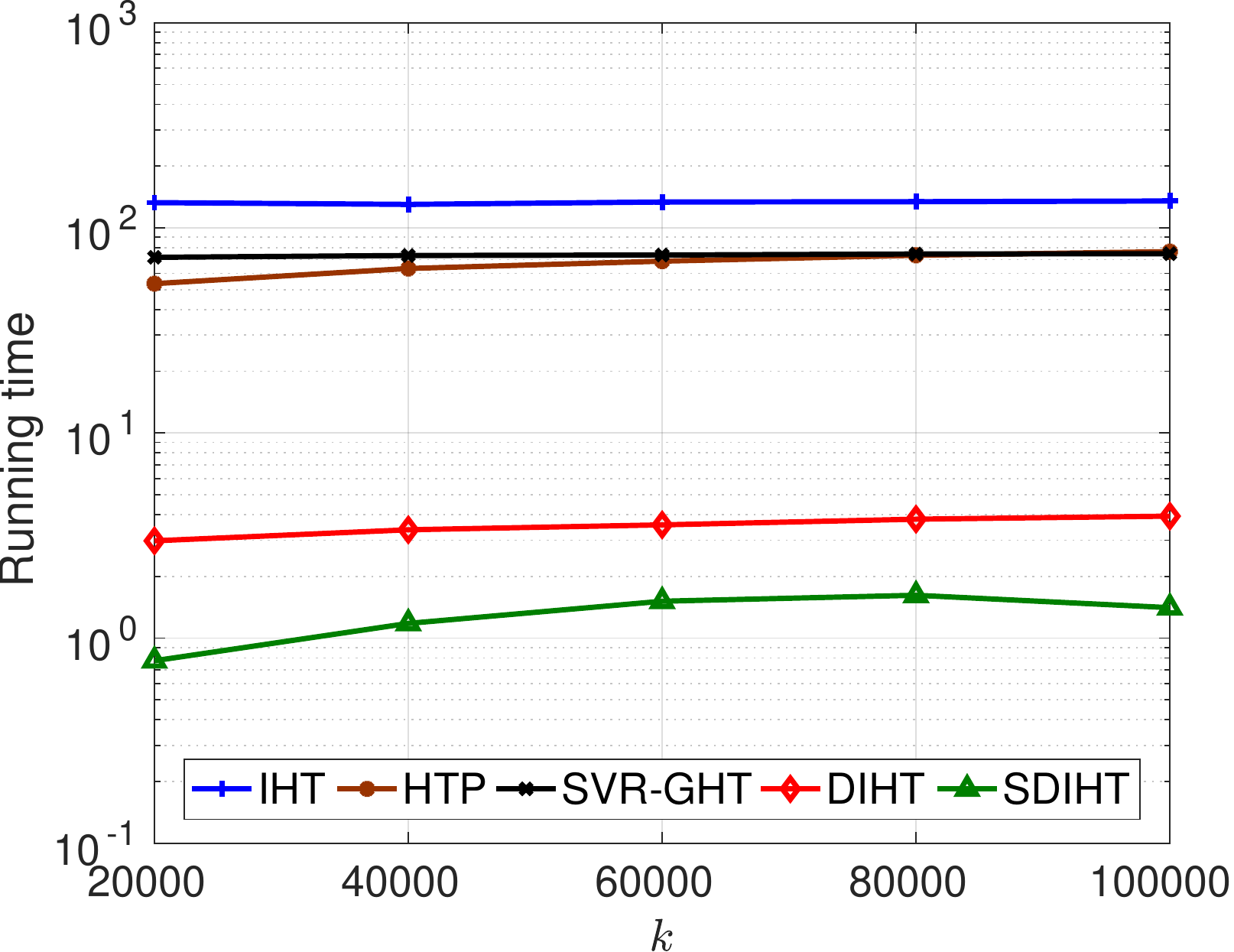}
\label{fig:TimeSDIHTNewSVM}
}
\caption{Huber loss: Running time (in second) comparison between the considered algorithms.}
\label{fig:TimeHuber}
\end{figure*}
\begin{figure*}
\centering
\subfigure[DIHT on RCV1]{
\includegraphics[width=0.21\textwidth,height=0.2\textwidth]{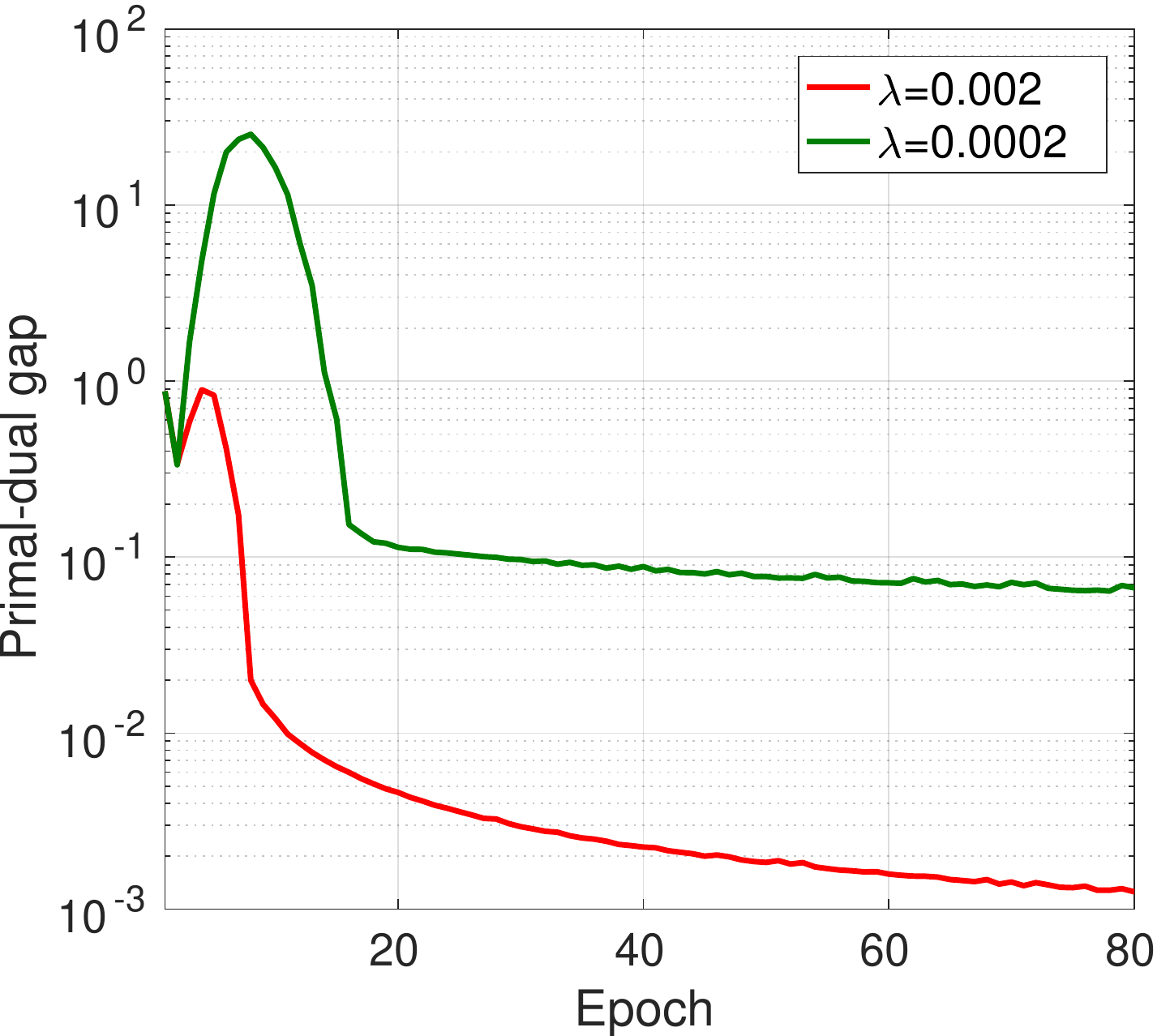}
\label{fig:PDGapDIHTRCVSVM}
}
\subfigure[SDIHT on RCV1]{
\includegraphics[width=0.21\textwidth,height=0.2\textwidth]{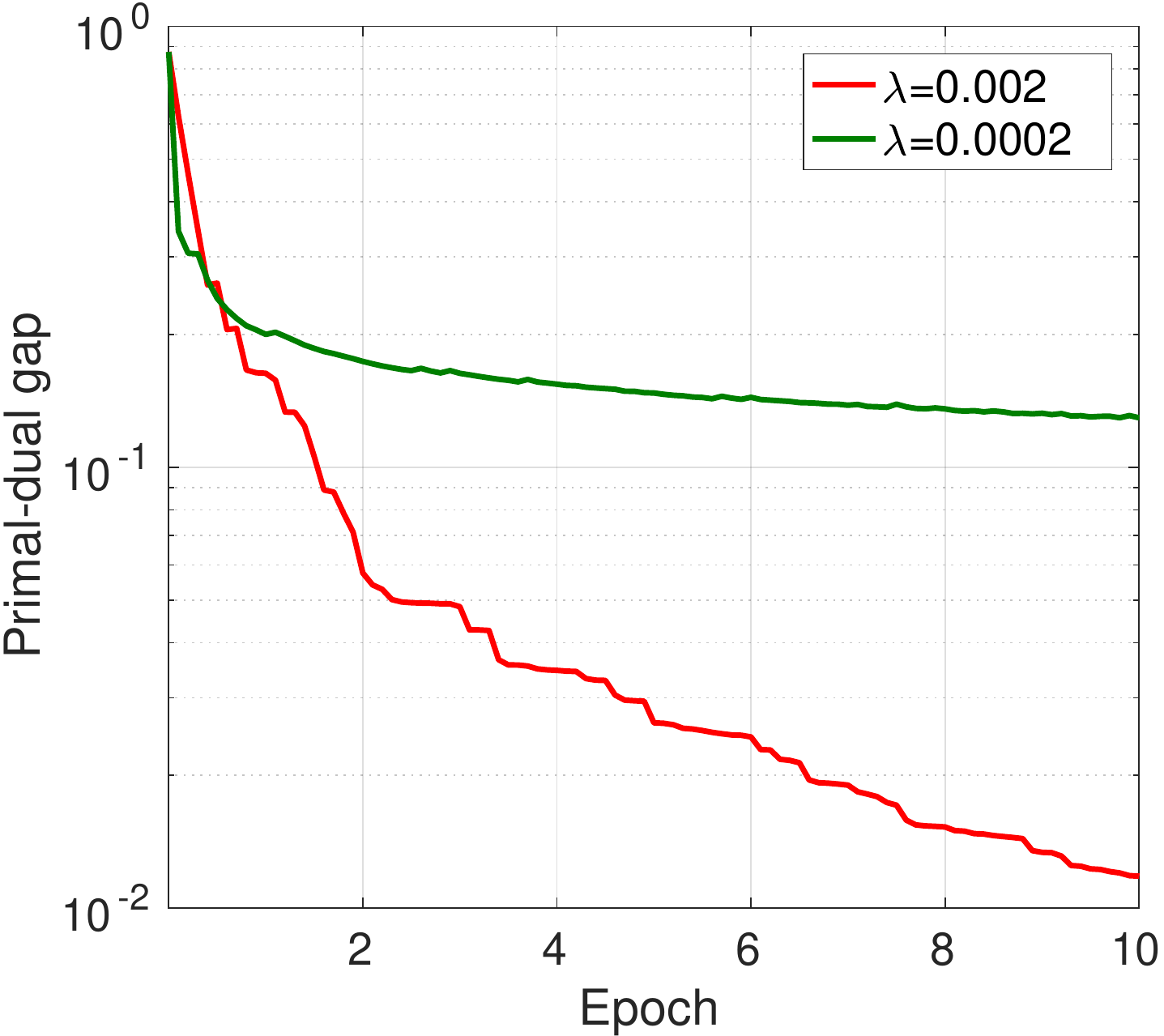}
\label{fig:PDGapSDIHTRCVSVM}
}
\subfigure[DIHT on News20]{
\includegraphics[width=0.21\textwidth,height=0.2\textwidth]{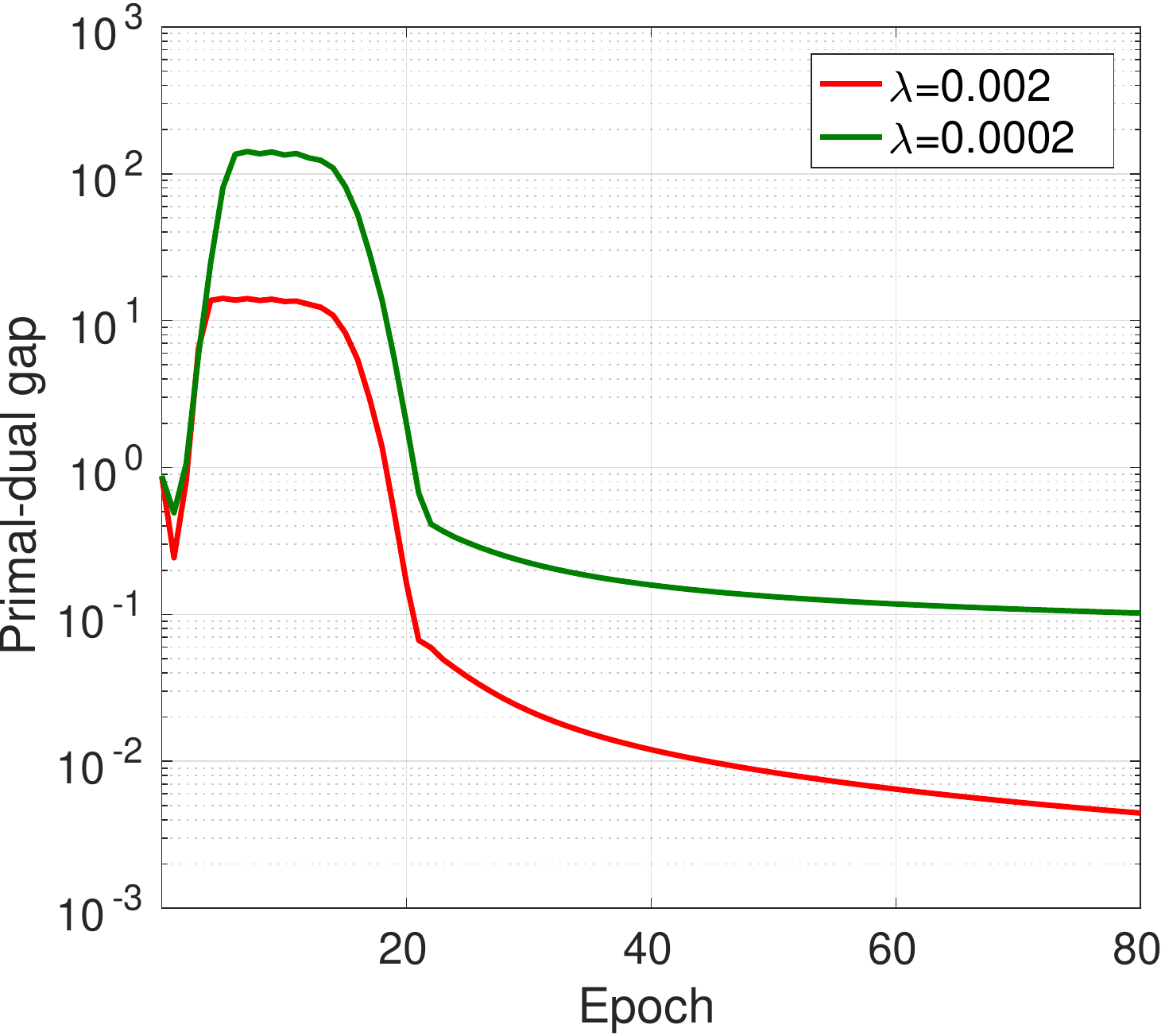}
\label{fig:PDGapDIHTNewSVM}
}
\subfigure[SDIHT on News20]{
\includegraphics[width=0.21\textwidth,height=0.2\textwidth]{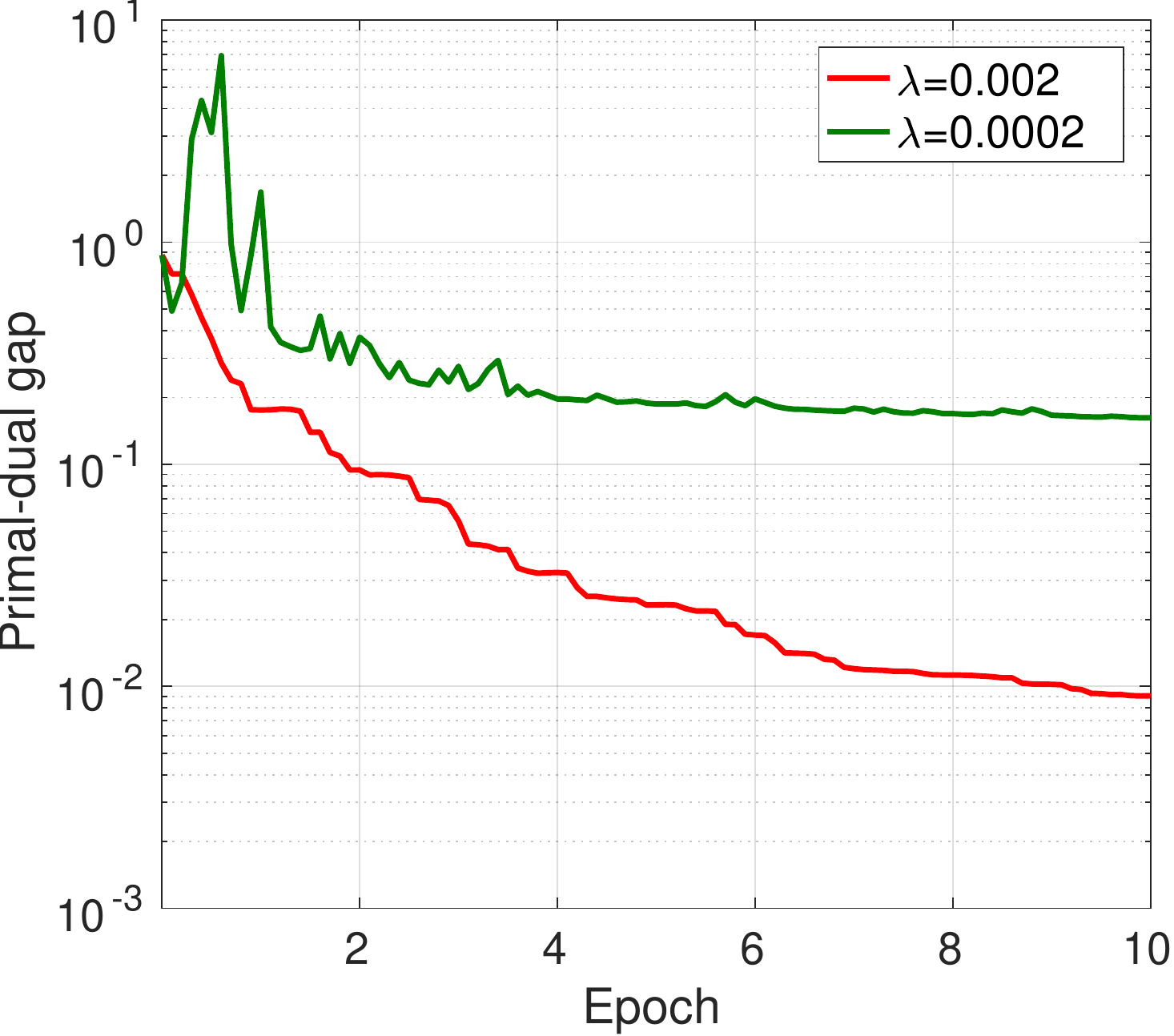}
\label{fig:PDGapSDIHTNewSVM}
}
\caption{Huber loss: The primal-dual gap evolving curves of DIHT and SDIHT. $k=600$ for RCV1 and $k=60000$ for News20.}
\label{fig:PDGapHuber}
\end{figure*}
\subsection{Model training efficiency evaluation}
\subsubsection{Huber loss model learning}
\label{subsubsect:huber}
We now evaluate the considered algorithms on the following $\ell_2$-regularized sparse Huber loss minimization problem:
\begin{equation}
\label{eqn:smoothHinge}
\min\limits_{\|w\|_0\le k} \frac{1}{N}\sum\limits_{i=1}^N l_{Huber}(y_ix_i^\top w)+\frac{\lambda}{2} \|w\|^2,
\end{equation}
where
\[
l_{Huber}(y_ix_i^\top w)=\left\{\begin{array}{ll} 0 & y_ix_i^\top w\ge 1 \\
1-y_ix_i^\top w-\frac{\gamma}{2} &  y_ix_i^\top w<1-\gamma \\
\frac{1}{2\gamma} (1-y_ix_i^\top w)^2 & \text{otherwise}
\end{array}\right..
\]
It is known that~\cite{shalev2013stochastic}
\[
l^*_{Huber}(\alpha_i)=\left\{\begin{array}{ll}  y_i\alpha_i+\frac{\gamma}{2}\alpha_i^2 & \text{if}~~y_i\alpha_i\in [-1,0] \\
+\infty & \text{otherwise}\end{array}\right..
\]
Two binary benchmark datasets from LibSVM data repository\footnote{\url{https://www.csie.ntu.edu.tw/~cjlin/libsvmtools/datasets/binary.html}}, RCV1 ($d=47,236$) and News20 ($d=1,355,191$), are used for algorithm efficiency evaluation and comparison. We select 0.5 million samples from RCV1 dataset for model training ($N\gg d$). For news20 dataset, all of the $19,996$ samples are used as training data ($d\gg N$).

We evaluate the algorithm efficiency of DIHT and SDIHT by comparing their running time against three primal baseline algorithms: IHT, HTP and gradient hard thresholding with stochastic variance reduction (SVR-GHT)~\cite{li2016stochastic}. We first run IHT by setting its convergence criterion to be $\frac{|P(w^{(t)})-P(w^{(t-1)})|}{P(w^{(t)})}\le 10^{-4}$ or maximum number of iteration is reached. After that we test the time cost spend by other algorithms to make the primal loss reach $P(w^{(t)})$. The parameter update step-size of all the considered algorithms is tuned by grid search. The parameter $\gamma$ is set to be 0.25. For the two stochastic algorithms SDIHT and SVR-GHT we randomly partition the training data into $|B|=10$ mini-batches.

Figure~\ref{fig:TimeHuber} shows the running time curves on both datasets under varying sparsity level $k$ and regularization strength $\lambda=0.002,0.0002$. It is obvious that under all tested $(k,\lambda)$ configurations on both datasets, DIHT and SDIHT need much less time than the primal baseline algorithms, IHT, HTP and SVR-GHT to reach the same primal sub-optimality. Figure~\ref{fig:PDGapHuber} shows the primal-dual gap convergence curves with respect to the number of epochs. This group of results support the theoretical prediction in Theorem~\ref{thrm:DIHT_Conv} and~\ref{thrm:SDIHT_Conv} that $\epsilon_{PD}$ converges non-asymptotically.
\begin{figure*}
\centering
\subfigure[RCV1, $\lambda=0.002$]{
\includegraphics[width=0.21\textwidth,height=0.2\textwidth]{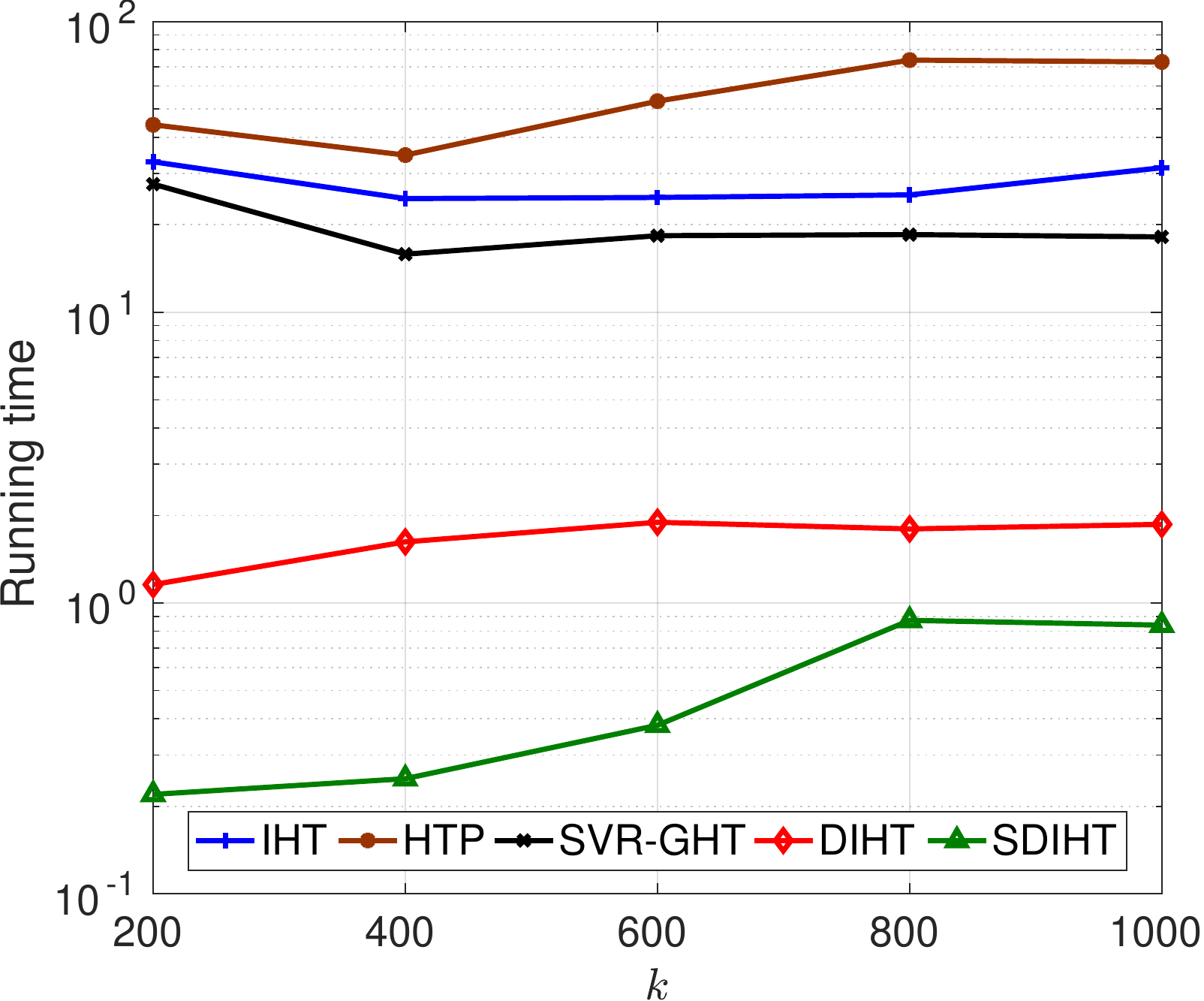}
\label{fig:PDGapDIHTRCVSVM}
}
\subfigure[RCV1, $\lambda=0.0002$]{
\includegraphics[width=0.21\textwidth,height=0.2\textwidth]{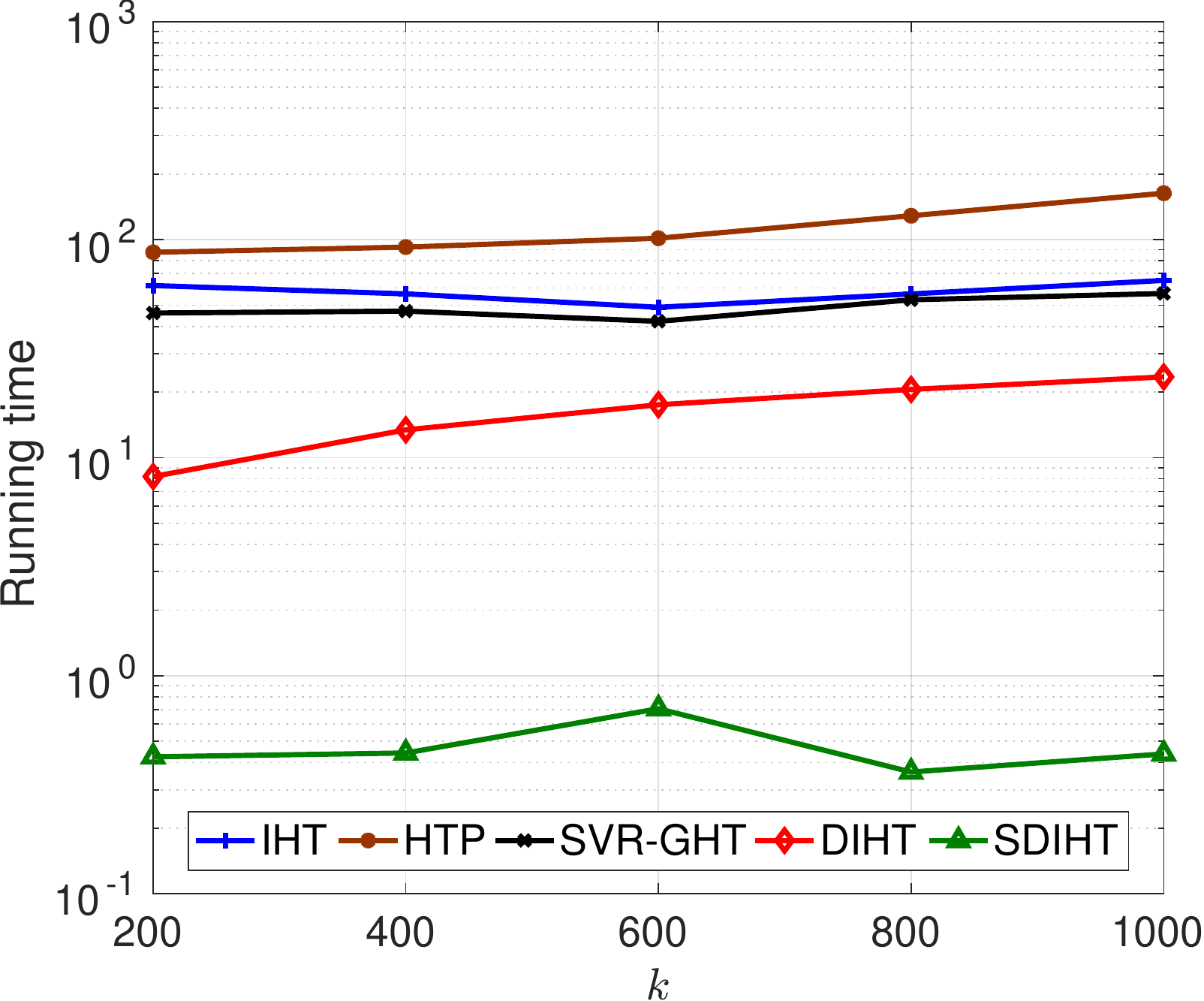}
\label{fig:PDGapSDIHTRCVSVM}
}
\subfigure[News20, $\lambda=0.002$]{
\includegraphics[width=0.21\textwidth,height=0.2\textwidth]{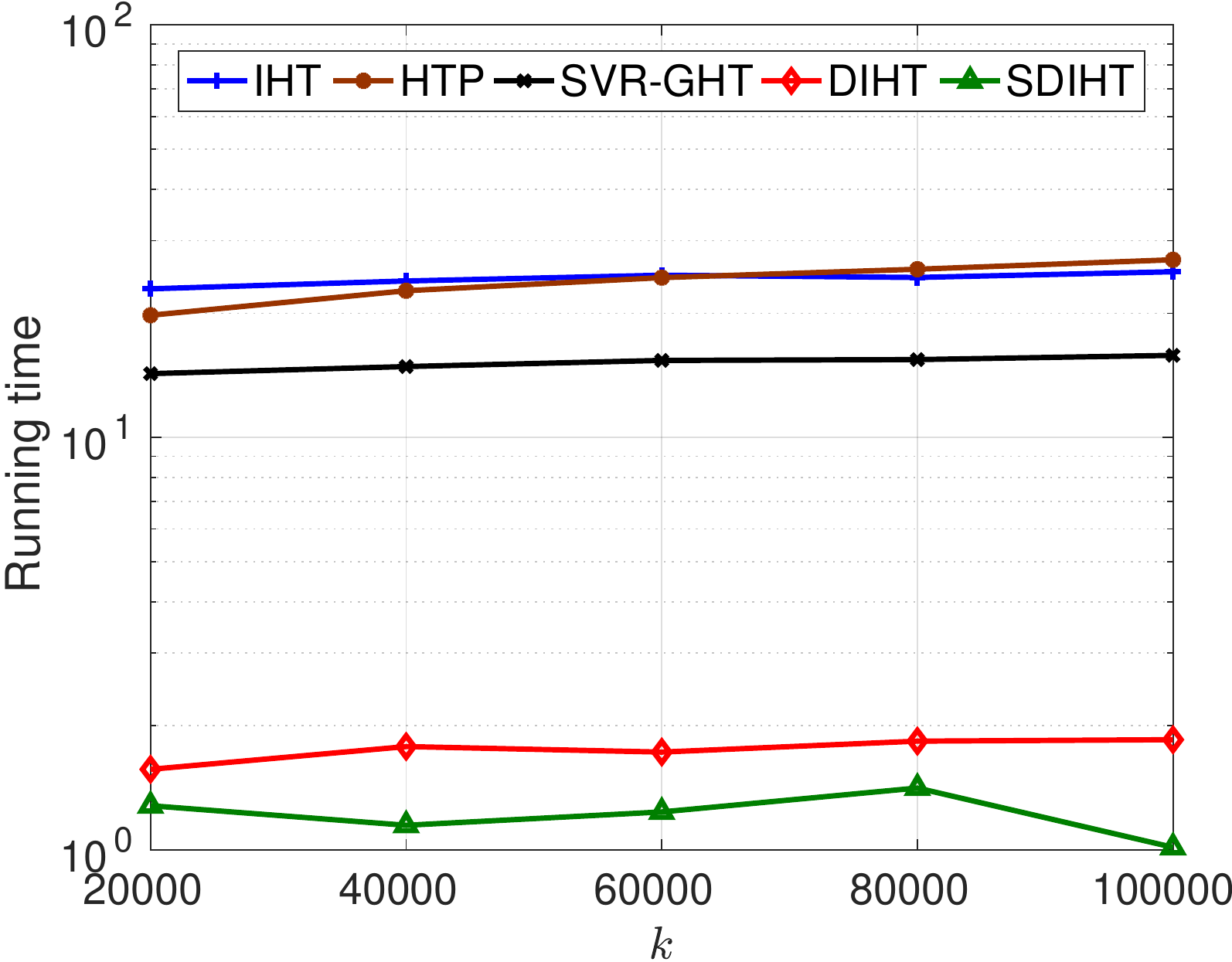}
\label{fig:PDGapDIHTNewSVM}
}
\subfigure[News20, $\lambda=0.0002$]{
\includegraphics[width=0.21\textwidth,height=0.2\textwidth]{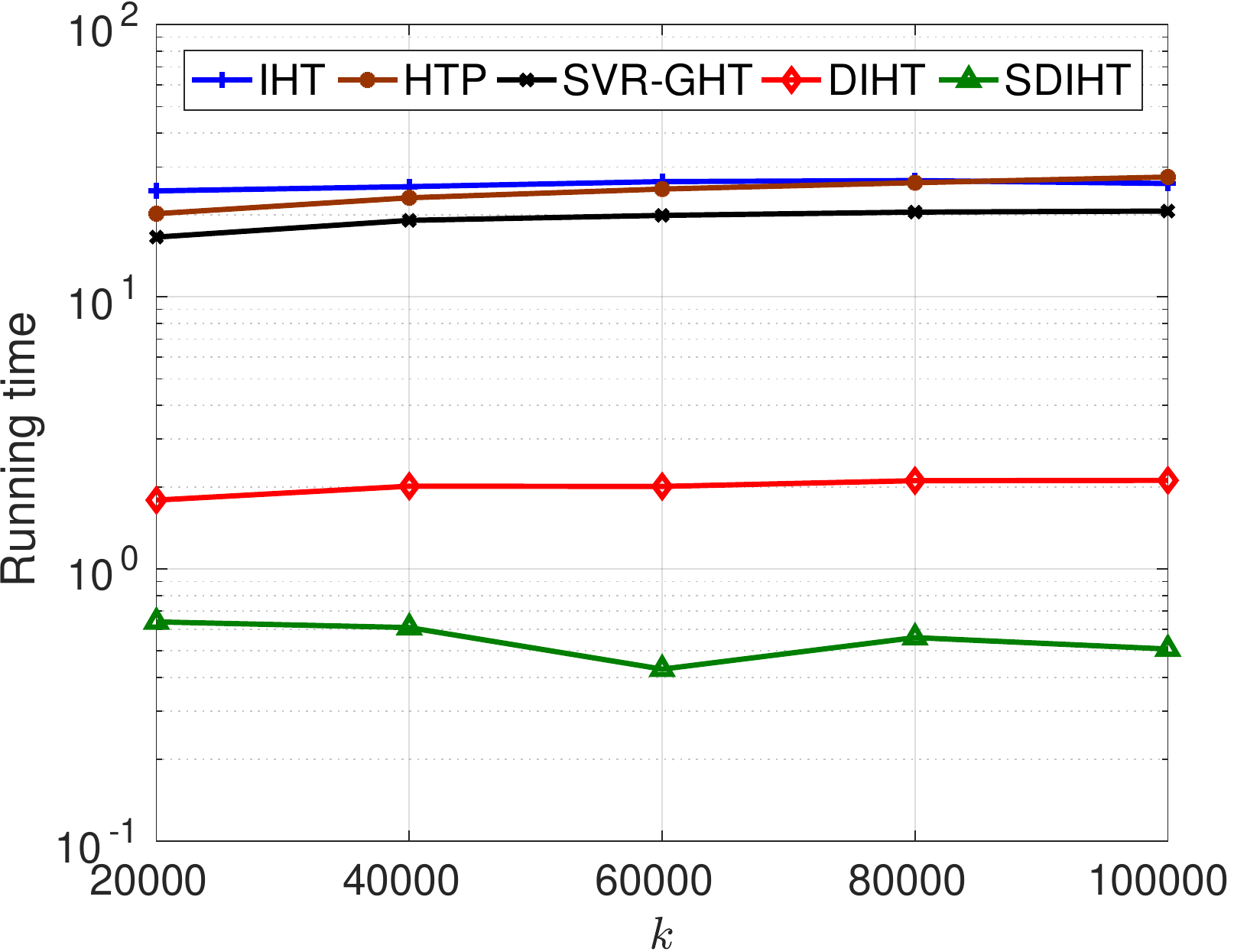}
\label{fig:PDGapSDIHTNewSVM}
}
\caption{Hinge loss: Running time (in second) comparison between the considered algorithms.}
\label{fig:TimeSVM}
\end{figure*}
\begin{figure*}
\centering
\subfigure[DIHT on RCV1]{
\includegraphics[width=0.21\textwidth,height=0.2\textwidth]{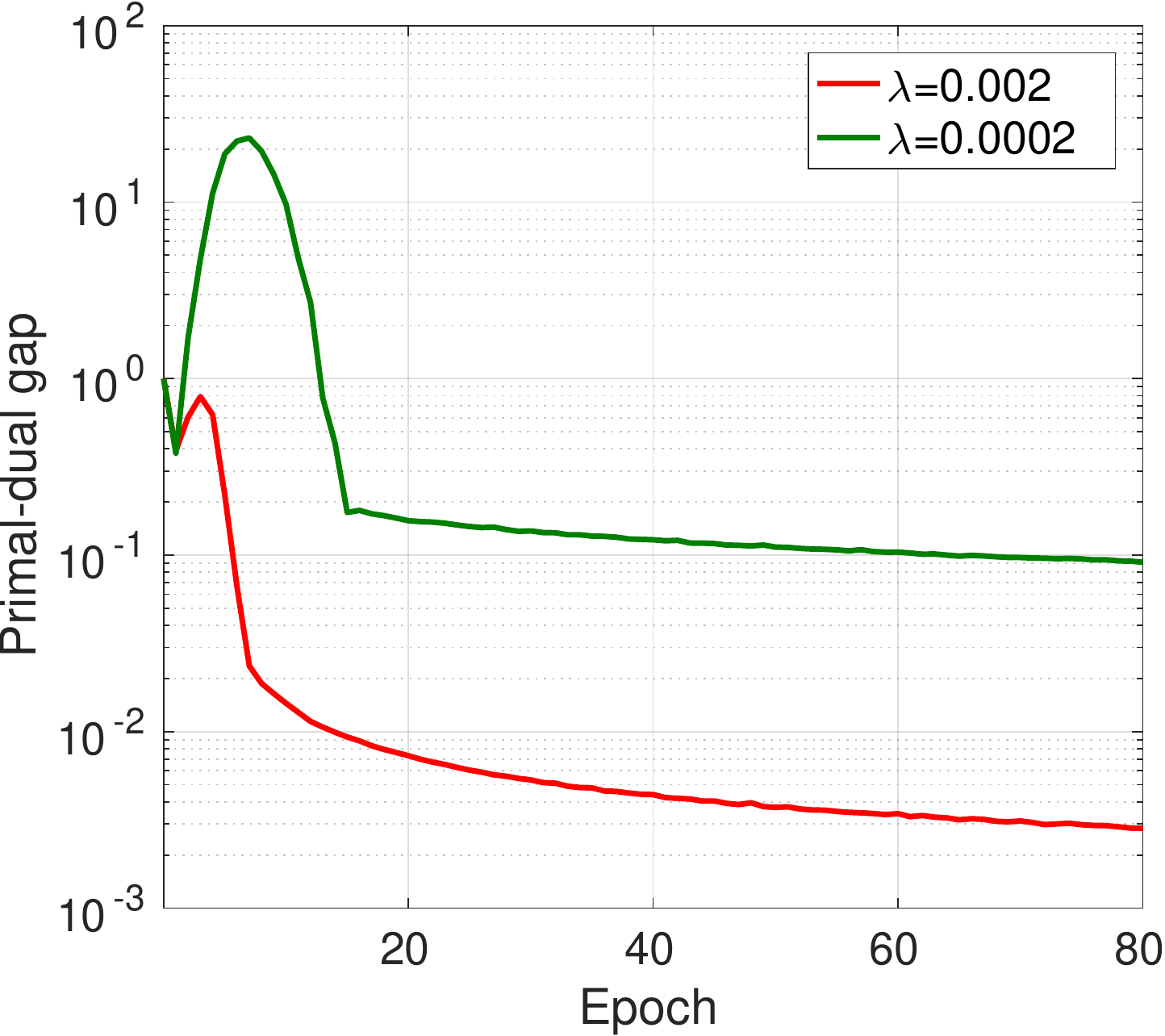}
\label{fig:TimeDIHTRCVSVM}
}
\subfigure[SDIHT on RCV1]{
\includegraphics[width=0.21\textwidth,height=0.2\textwidth]{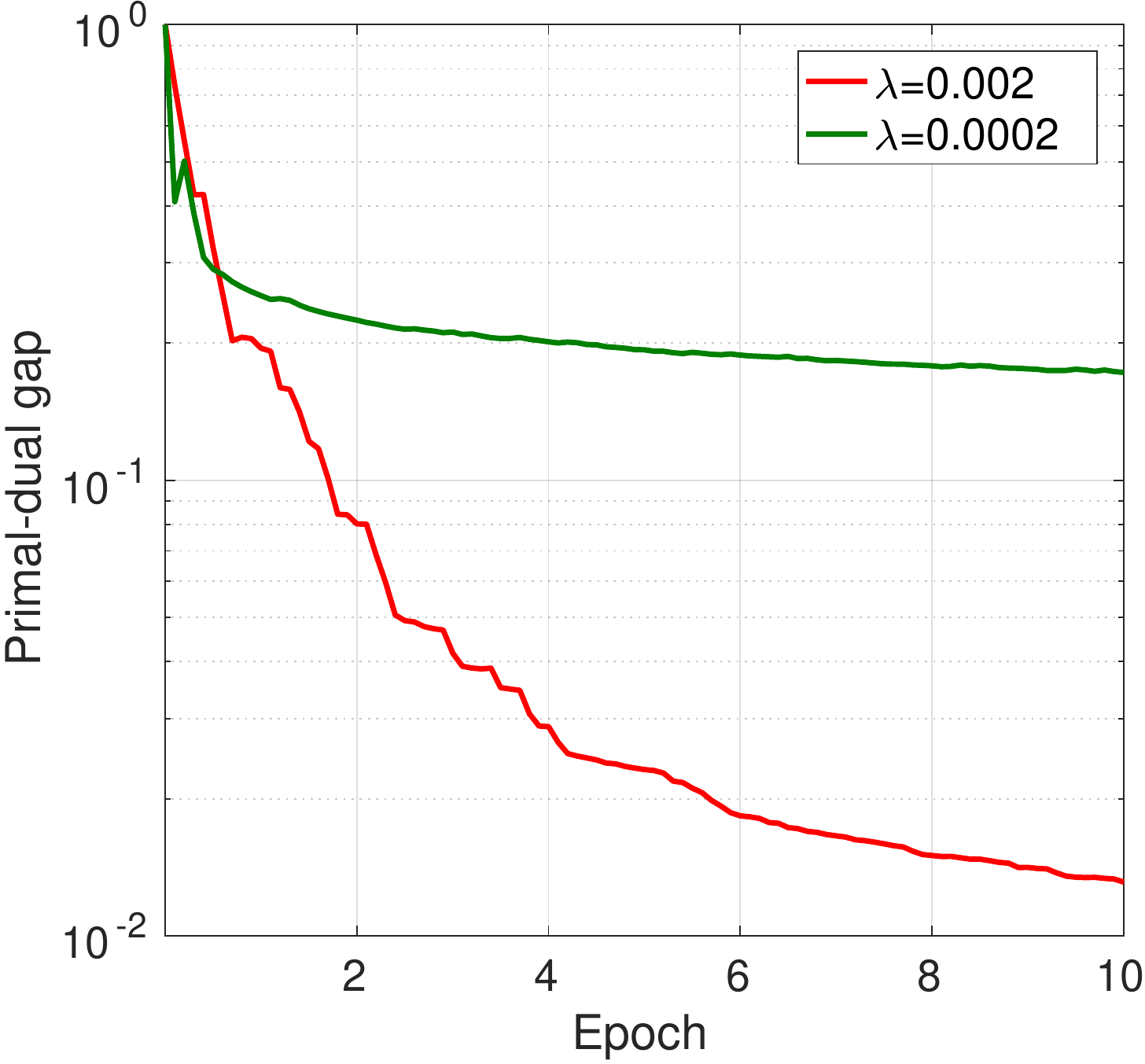}
\label{fig:TimeSDIHTRCVSVM}
}
\subfigure[DIHT on News20]{
\includegraphics[width=0.21\textwidth,height=0.2\textwidth]{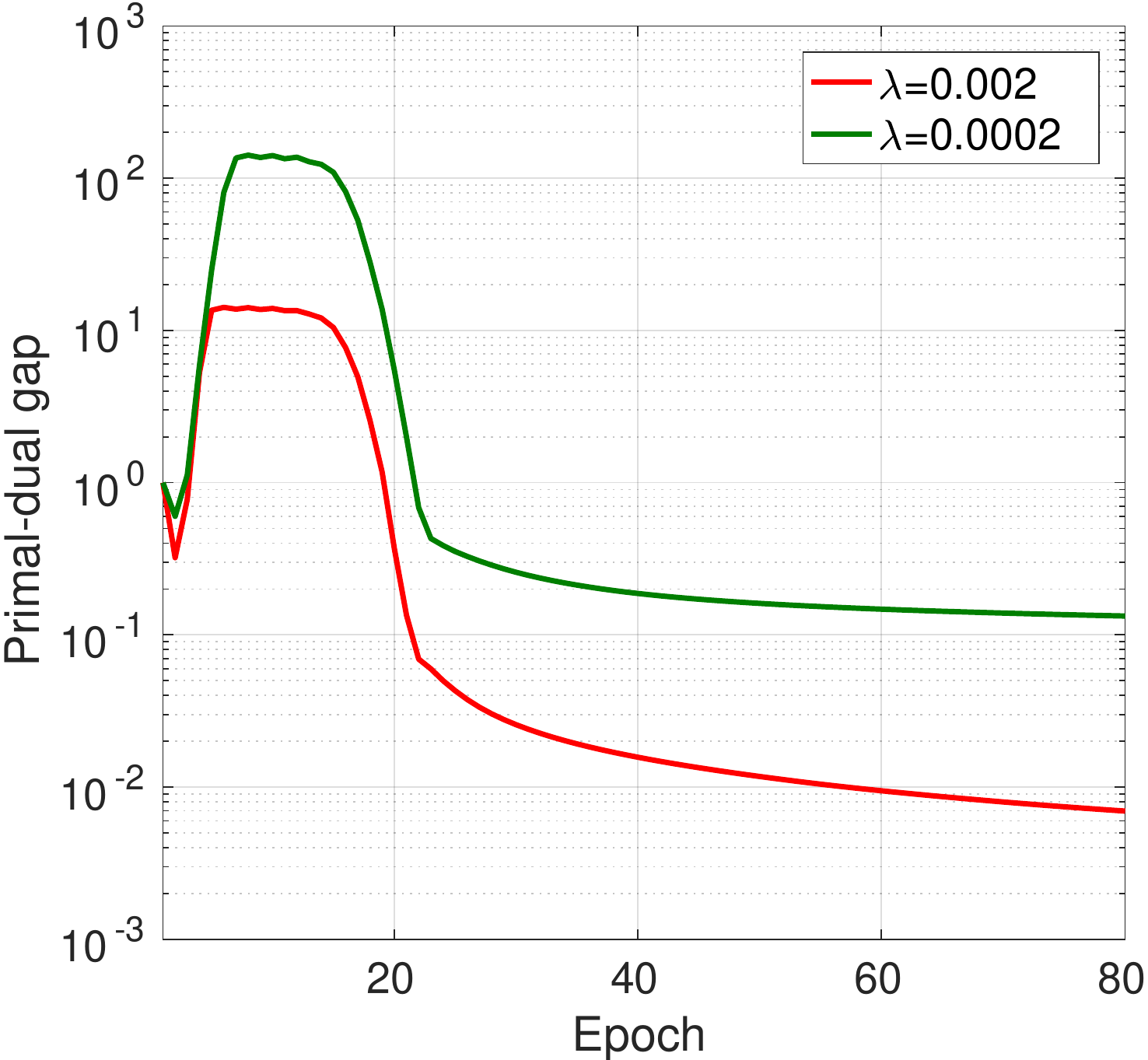}
\label{fig:TimeDIHTNewSVM}
}
\subfigure[SDIHT on News20]{
\includegraphics[width=0.21\textwidth,height=0.2\textwidth]{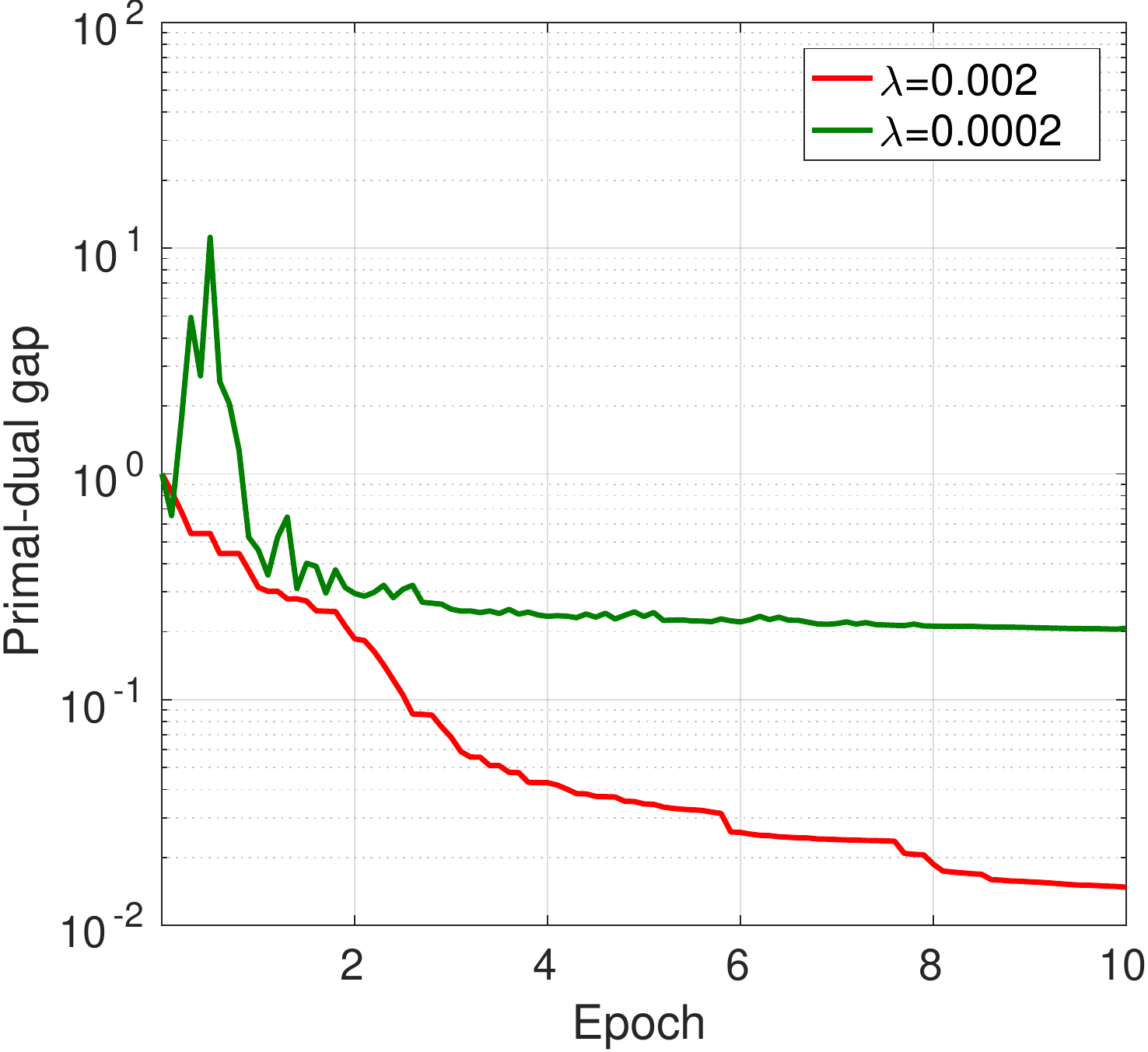}
\label{fig:TimeSDIHTNewSVM}
}
\caption{Hinge loss: The primal-dual gap evolving curves of DIHT and SDIHT. $k=600$ for RCV1 and $k=60000$ for News20.}
\label{fig:PDGapSVM}
\end{figure*}
\subsubsection{Hinge loss model learning}
We further test the performance of our algorithms when applied to the following $\ell_2$-regularized sparse hinge loss minimization problem:
\[
\min\limits_{\|w\|_0\le k} \frac{1}{N}\sum\limits_{i=1}^N l_{Hinge}(y_ix_i^{\top}w)+\frac{\lambda}{2}\|w\|^2,
\]
where $l_{Hinge}(y_ix_i^{\top}w)=\max(0,1-y_iw^{\top}x_i)$. It is standard to know~\cite{hsieh2008dual} \[l^{*}_{Hinge}(\alpha_i)=\left\{\begin{array}{ll}  y_i\alpha_i & \text{if}~~y_i\alpha_i\in [-1,0] \\
+\infty & \text{otherwise}\end{array}\right..\]
We follow the same experiment protocol as in \S~\ref{subsubsect:huber} to compare the considered algorithms on the benchmark datasets. The time cost comparison is illustrated in Figure~\ref{fig:TimeSVM} and the prima-dual gap sub-optimality is illustrated in Figure~\ref{fig:PDGapSVM}. This group of results indicate that DIHT and SDIHT still exhibit remarkable efficiency advantage over the considered primal IHT algorithms even when the loss function is non-smooth.
\vspace{-2mm}
\section{Conclusion}
\label{sect:conclusion}
In this paper, we systematically investigate duality theory and algorithms for solving the sparsity-constrained minimization problem which is NP-hard and non-convex in its primal form. As a theoretical contribution, we develop a sparse Lagrangian duality theory which guarantees strong duality in sparse settings, under mild sufficient and necessary conditions. This theory opens the gate to solve the original NP-hard/non-convex problem equivalently in a dual space. We then propose DIHT as a first-order method to maximize the non-smooth dual concave formulation. The algorithm is characterized by dual super-gradient ascent and primal hard thresholding. To further improve iteration efficiency in large-scale settings, we propose SDIHT as a block stochastic variant of DIHT. For both algorithms we have proved sub-linear primal-dual gap convergence rate when the primal loss is smooth, without assuming RIP-style conditions. Based on our theoretical findings and numerical results, we conclude that DIHT and SDIHT are theoretically sound and computationally attractive alternatives to the conventional primal IHT algorithms, especially when the sample size is smaller than feature dimensionality.
\vspace{-3mm}
\section*{Acknowledgements}
Xiao-Tong Yuan is supported in part by Natural Science Foundation of China (NSFC) under Grant 61402232, Grant 61522308, and in part by Natural Science Foundation of Jiangsu Province of China (NSFJPC) under Grant BK20141003. Qingshan Liu is supported in part by NSFC under Grant 61532009.

\newpage
\bibliography{egbib}
\newpage

\onecolumn
\appendix
\section{Technical Proofs}

\subsection{Proof of Theorem~\ref{thrm:sparse_saddle_point}}
\label{append:proof_sparse_saddle_point}
\begin{proof}
``$\Leftarrow$'': If the pair $(\bar w, \bar \alpha)$ is a sparse saddle point for $L$, then from the definition of conjugate convexity and inequality~\eqref{inequat:sparse_saddle_point} we have
\[
\begin{aligned}
P(\bar w) =\max_{\alpha \in \mathcal{F}^N} L(\bar w, \alpha) \le L(\bar w, \bar \alpha) \le \min_{\|w\|_0 \le k} L(w,\bar \alpha).
\end{aligned}
\]
On the other hand, we know that for any $\|w\|_0\le k$ and $\alpha \in \mathcal{F}^N$
\[
L(w,\alpha) \le \max_{\alpha' \in \mathcal{F}^N} L(w, \alpha') = P(w).
\]
By combining the preceding two inequalities we obtain
\[
P(\bar w) \le \min_{\|w\|_0 \le k} L(w,\bar \alpha) \le \min_{\|w\|_0\le k} P(w) \le P(\bar w).
\]
Therefore $P(\bar w)= \min_{\|w\|_0\le k} P(w)$, i.e., $\bar w$ solves the problem in~\eqref{prob:general}, which proves the necessary condition~(a). Moreover, the above arguments lead to
\[
P(\bar w) =\max_{\alpha \in \mathcal{F}^N} L(\bar w, \alpha) = L(\bar w, \bar \alpha).
\]
Then from the maximizing argument property of convex conjugate we know that $\bar\alpha_i \in \partial l_i(\bar w^\top x_i)$. Thus the necessary condition~(b) holds. Note that
\begin{equation}\label{equat:L_quadratic}
L(w,\bar \alpha) = \frac{\lambda}{2}\left\| w + \frac{1}{N\lambda}\sum_{i=1}^N \bar\alpha_i x_i\right\|^2 - \frac{1}{N}\sum_{i=1}^N l^*_i(\bar \alpha_i)+ C,
\end{equation}
where $C$ is a quantity not dependent on $w$. Let $\bar F=\supp(\bar w)$. Since the above analysis implies $L(\bar w, \bar \alpha) = \min_{\|w\|_0 \le k} L(w,\bar \alpha)$, it must hold that
\[
\bar w = \mathrm{H}_{\bar F}\left(-\frac{1}{N\lambda} \sum_{i=1}^N \bar\alpha_i x_i \right)=\mathrm{H}_k\left(-\frac{1}{N\lambda} \sum_{i=1}^N \bar\alpha_i x_i \right).
\]
This validates the necessary condition~(c).

``$\Rightarrow$'': Conversely, let us assume that $\bar w$ is a $k$-sparse solution to the problem~\eqref{prob:general} (i.e., conditio(a)) and let $\bar \alpha_i\in \partial l_i(\bar w^\top x_i)$ (i.e., condition~(b)). Again from the maximizing argument property of convex conjugate we know that $l_i(\bar w^\top x_i) = \bar \alpha_i  \bar w^\top x_i - l^*_i(\bar \alpha_i)$. This leads to
\begin{equation}\label{inequat:theorem_key_1}
L(\bar w, \alpha) \le P(\bar w) = \max_{\alpha \in \mathcal{F}^N} L(\bar w, \alpha) = L(\bar w, \bar\alpha).
\end{equation}
The sufficient condition~(c) guarantees that $\bar F$ contains the top $k$ (in absolute value) entries of $-\frac{1}{N\lambda} \sum_{i=1}^N \bar\alpha_i x_i$. Then based on the expression in~\eqref{equat:L_quadratic} we can see that the following holds for any $k$-sparse vector $w$
\begin{equation}\label{inequat:theorem_key_2}
L(\bar w, \bar\alpha) \le L(w,\bar \alpha).
\end{equation}
By combining the inequalities~\eqref{inequat:theorem_key_1} and~\eqref{inequat:theorem_key_2} we get that for any $\|w\|_0\le k$ and $\alpha \in \mathcal{F}^N$,
\[
L(\bar w, \alpha) \le L(\bar w, \bar \alpha) \le L(w, \bar \alpha).
\]
This shows that $(\bar w, \bar \alpha)$ is a sparse saddle point of the Lagrangian $L$.
\end{proof}

\subsection{Proof of Theorem~\ref{thrm:sparse_mini_max}}
\label{append:proof_sparse_mini_max}
\begin{proof}
``$\Rightarrow$'': Let $(\bar w, \bar\alpha)$ be a saddle point for $L$. On one hand, note that the following holds for any $k$-sparse $w'$ and $\alpha' \in \mathcal{F}^N$
\[
\min_{\|w\|_0 \le k} L(w, \alpha') \le L(w',\alpha') \le \max_{\alpha \in \mathcal{F}^N}L(w', \alpha),
\]
which implies
\begin{equation}\label{inequat:proof_minimax_key_1}
\max_{\alpha \in \mathcal{F}^N} \min_{\|w\|_0\le k} L(w, \alpha) \le \min_{\|w\|_0\le k}\max_{\alpha \in \mathcal{F}^N} L(w,\alpha).
\end{equation}
On the other hand, since $(\bar w, \bar \alpha)$ is a saddle point for $L$, the following is true:
\begin{equation}\label{inequat:proof_minimax_key_2}
\begin{aligned}
\min_{\|w\|_0 \le k} \max_{\alpha \in \mathcal{F}^N} L(w,\alpha) &\le \max_{\alpha \in \mathcal{F}^N} L(\bar w, \alpha) \\
&\le L(\bar w, \bar \alpha) \le \min_{\|w\|_0\le k} L(w, \bar \alpha) \\
&\le \max_{\alpha \in \mathcal{F}^N}\min_{\|w\|_0\le k} L(w,\alpha).
\end{aligned}
\end{equation}
By combining~\eqref{inequat:proof_minimax_key_1} and~\eqref{inequat:proof_minimax_key_2} we prove the equality in~\eqref{equat:mini-max}.

``$\Leftarrow$'': Assume that the equality in~\eqref{equat:mini-max} holds. Let us define $\bar w$ and $\bar \alpha$ such that
\[
\begin{aligned}
\max_{\alpha \in \mathcal{F}^N}L(\bar w, \alpha) &= \min_{\|w\|_0 \le k}\max_{\alpha \in \mathcal{F}^N}L(w, \alpha) \\
\min_{\|w\|_0 \le k}L(w, \bar\alpha) &=\max_{\alpha \in \mathcal{F}^N}\min_{\|w\|_0 \le k}L(w, \alpha)
\end{aligned}.
\]
Then we can see that for any $\alpha \in \mathcal{F}^N$,
\[
L(\bar w, \bar\alpha) \ge \min_{\|w\|_0\le k} L(w,\bar\alpha) = \max_{\alpha' \in \mathcal{F}^N}L(\bar w, \alpha') \ge L(\bar w, \alpha),
\]
where the ``$=$'' is due to~\eqref{equat:mini-max}. In the meantime, for any $\|w\|_0\le k$,
\[
L(\bar w, \bar\alpha) \le \max_{\alpha \in \mathcal{F}^N}L(\bar w, \alpha) = \min_{\|w'\|_0\le k} L(w',\bar\alpha) \le L(w, \bar \alpha).
\]
This shows that $(\bar w,\bar\alpha)$ is a sparse saddle point for $L$.
\end{proof}

\subsection{Proof of Lemma~\ref{lemma:concave}}
\label{append:proof_concave}
\begin{proof}
For any fixed $\alpha \in \mathcal{F}^N$, then it is easy to verify that the $k$-sparse minimum of $L(w,\alpha)$ with respect to $w$ is attained at the following point:
\[
w(\alpha) = \argmin_{\|w\|_0 \le k} L(w, \alpha) =\mathrm{H}_k\left(-\frac{1}{N\lambda} \sum_{i=1}^N \alpha_i x_i \right).
\]
Thus we have
\[
\begin{aligned}
D(\alpha) &= \min_{\|w\|_0 \le k} L(w, \alpha) = L(w(\alpha), \alpha)\\
&= \frac{1}{N}\sum\limits_{i=1}^N\left(\alpha_iw(\alpha)^\top x_i - l^*_i(\alpha_i)\right) + \frac{\lambda}{2}\|w(\alpha)\|^2 \\
& \overset{\zeta_1}{=} \frac{1}{N}\sum_{i=1}^N - l^*_i(\alpha_i) - \frac{\lambda}{2}\|w(\alpha)\|^2,
\end{aligned}
\]
where ``$\zeta_1$'' follows from the above definition of $w(\alpha)$.

Now let us consider two arbitrary dual variables $\alpha', \alpha'' \in \mathcal{F}^N$ and any $g(\alpha'')\in\frac{1}{N}[w(\alpha'')^\top x_1 - \partial l^*_1(\alpha''_1),...,w(\alpha'')^\top x_N - \partial l^*_N(\alpha''_N)]$. From the definition of $D(\alpha)$ and the fact that $L(w,\alpha)$ is concave with respect to $\alpha$ at any fixed $w$ we can derive that
\[
\begin{aligned}
D(\alpha') &= L(w(\alpha'),\alpha') \\
& \le L(w(\alpha''), \alpha') \\
& \le L(w(\alpha''), \alpha'') + \left\langle g(\alpha'') , \alpha' - \alpha''\right\rangle.
\end{aligned}
\]
This shows that $D(\alpha)$ is a concave function and its super-differential is as given in the theorem.

If we further assume that $w(\alpha)$ is unique and $\{l^*_i\}_{i=1,...,N}$ are differentiable at any $\alpha$, then $\partial D(\alpha)=\frac{1}{N} [w(\alpha)^\top x_1 - \partial l^*_1(\alpha_1),...,w(\alpha)^\top x_N - \partial l^*_N(\alpha_N)]$ becomes unique, which implies that $\partial D(\alpha)$ is the unique super-gradient of $D(\alpha)$.
\end{proof}

\subsection{Proof of Theorem~\ref{thrm:sparse_strong_duality}}
\label{append:proof_sparse_strong_duality}
\begin{proof}
``$\Rightarrow$'': Given the conditions in the theorem, it can be known from Theorem~\ref{thrm:sparse_saddle_point} that the pair $(\bar w, \bar\alpha)$ forms a sparse saddle point of $L$. Thus based on the definitions of sparse saddle point and dual function $D(\alpha)$ we can show that
\[
D(\bar \alpha) = \min_{\|w\|_0\le k} L(w, \bar\alpha) \ge L(\bar w, \bar\alpha) \ge L(\bar w, \alpha)\ge D(\alpha).
\]
This implies that $\bar\alpha$ solves the dual problem in~\eqref{prob:dual}. Furthermore, Theorem~\ref{thrm:sparse_mini_max} guarantees the following
\[
D(\bar\alpha) = \max_{\alpha \in \mathcal{F}^N} \min_{\|w\|_0\le k} L(w, \alpha) =\min_{\|w\|_0\le k}\max_{\alpha \in \mathcal{F}^N} L(w,\alpha) = P(\bar w).
\]
This indicates that the primal and dual optimal values are equal to each other.

``$\Leftarrow$'': Assume that $\bar\alpha$ solves the dual problem in~\eqref{prob:dual} and $D(\bar\alpha) = P(\bar w)$. Since $D(\bar\alpha)\le P(w)$ holds for any $\|w\|_0\le k$, $\bar w$ must be the sparse minimizer of $P(w)$. It follows that
\[
\max_{\alpha \in \mathcal{F}^N} \min_{\|w\|_0\le k} L(w, \alpha) = D(\bar\alpha) =P(\bar w) = \min_{\|w\|_0\le k}\max_{\alpha \in \mathcal{F}^N} L(w,\alpha).
\]
From the ``$\Leftarrow$'' argument in the proof of Theorem~\ref{thrm:sparse_mini_max} and Corollary~\ref{corol:sparse_mini_max} we get that the conditions~(a)$\sim$(c) in Theorem~\ref{thrm:sparse_saddle_point} should be satisfied for $(\bar w, \bar\alpha)$.
\end{proof}

\subsection{Proof of Theorem~\ref{thrm:DIHT_Conv}}
\label{append:proof_DIHT_conv}

We need a series of technical lemmas to prove this theorem. The following lemmas shows that under proper conditions, $w(\alpha)$ is locally smooth around $\bar w = w(\bar \alpha)$.
\begin{lemma}\label{lemma:uniqueness_gap}
Let $X=[x_1,...,x_N]\in \mathbb{R}^{d\times N}$ be the data matrix. Assume that $\{l_i\}_{i=1,...,N}$ are differentiable and
\[
\bar\epsilon: = \bar w_{\min} - \frac{1}{\lambda} \|P'(\bar w)\|_\infty >0.
\]
If $\|\alpha - \bar\alpha\| \le \frac{\lambda N\bar\epsilon}{2\sigma_{\max}(X)} $, then $\supp(w(\alpha))=\supp(\bar w)$ and
\[
\|w(\alpha) - \bar w\| \le \frac{\sigma_{\max}(X,k)}{N\lambda}\|\alpha - \bar\alpha\|.
\]
\end{lemma}
\begin{proof}
For any $\alpha \in \mathcal{F}^N$, let us define
\[
\tilde w(\alpha) = -\frac{1}{N\lambda} \sum_{i=1}^N \alpha_i x_i.
\]
Consider $\bar F = \supp(\bar w)$. Given $\bar\epsilon >0$, it is known from Theorem~\ref{thrm:sparse_strong_duality} that $\bar w = \mathrm{H}_{\bar F} \left(\tilde w(\bar\alpha) \right)$ and $\frac{P'(\bar w)}{\lambda} = \mathrm{H}_{\bar F^c} \left(-\tilde w(\bar\alpha) \right)$. Then $\bar\epsilon >0$ implies $\bar F$ is unique, i.e., the top $k$ entries of $\tilde w(\bar\alpha)$ is unique. Given that $\|\alpha - \bar\alpha\| \le \frac{\lambda N\bar\epsilon}{2\sigma_{\max}(X)} $, it can be shown that
\[
\|\tilde w(\alpha) - \tilde w(\bar\alpha)\| = \frac{1}{N\lambda}\|X(\alpha - \bar\alpha)\| \le\frac{\sigma_{\max}(X)}{N\lambda} \|\alpha - \bar\alpha\| \le \frac{\bar\epsilon}{2}.
\]
This indicates that $\bar F$ still contains the (unique) top $k$ entries of $\tilde w(\alpha)$. Therefore,
\[
\supp(w(\alpha)) = \bar F = \supp(\bar w).
\]
Then it must hold that
\[
\begin{aligned}
\|w(\alpha) - w(\bar\alpha)\|  & =\|\mathrm{H}_{\bar F} \left(\tilde w(\alpha) \right) - \mathrm{H}_{\bar F} \left(\tilde w(\bar\alpha)\right)\| \\
&= \frac{1}{N\lambda}\|X_{\bar F}(\alpha - \bar\alpha)\| \\
&\le\frac{\sigma_{\max}(X, k)}{N\lambda} \|\alpha - \bar\alpha\|.
\end{aligned}
\]
This proves the desired bound.
\end{proof}

The following lemma bounds the estimation error $\|\alpha - \bar\alpha\|=O(\sqrt{\langle D'(\alpha), \bar\alpha - \alpha\rangle})$ when the primal loss $\{l_i\}_{i=1}^N$ are smooth.
\begin{lemma}\label{lemma:strong_concavity}
Assume that the primal loss functions $\{l_i(\cdot)\}_{i=1}^N$ are $1/\mu$-smooth. Then the following inequality holds for any $\alpha, \alpha'' \in \mathcal{F}$ and $g(\alpha'') \in \partial D(\alpha'')$:
\[
D(\alpha') \le D(\alpha'') + \langle g(\alpha'') , \alpha' - \alpha'' \rangle - \frac{\lambda N \mu+\sigma^2_{\min}(X,k)}{2\lambda N^2} \|\alpha' - \alpha''\|^2.
\]
Moreover, $\forall \alpha \in \mathcal{F}$ and $g(\alpha)\in \partial D(\alpha)$,
\[
\|\alpha - \bar\alpha\| \le \sqrt{\frac{2\lambda N^2\langle g(\alpha), \bar\alpha - \alpha\rangle}{\lambda N \mu+\sigma^2_{\min}(X,k)}}.
\]
\end{lemma}
\begin{proof}
Recall that
\[
D(\alpha) = \frac{1}{N}\sum_{i=1}^N- l^*_i(\alpha_i) - \frac{\lambda}{2}\|w(\alpha)\|^2,
\]
Now let us consider two arbitrary dual variables $\alpha', \alpha'' \in \mathcal{F}$. The assumption of $l_i$ being $1/\mu$-smooth implies that its convex conjugate function $l_i^*$ is $\mu$-strongly-convex. Let $F''=\supp(w(\alpha''))$. Then
\[
\begin{aligned}
D(\alpha')=& \frac{1}{N}\sum_{i=1}^N- l^*_i(\alpha'_i) - \frac{\lambda}{2}\|w(\alpha')\|^2\\
=& \frac{1}{N}\sum_{i=1}^N- l^*_i(\alpha'_i) - \frac{\lambda}{2}\left\|\mathrm{H}_k\left(-\frac{1}{N\lambda} \sum_{i=1}^N \alpha'_i x_i \right)\right\|^2\\
\le& \frac{1}{N}\sum_{i=1}^N \left(- l^*_i(\alpha''_i) - l^{*'}_i(\alpha''_i)(\alpha'_i - \alpha''_i) - \frac{\mu}{2}(\alpha'_i-\alpha''_i)^2 \right) - \frac{\lambda}{2}\left\|\mathrm{H}_{F''}\left(-\frac{1}{N\lambda} \sum_{i=1}^N \alpha'_i x_i \right)\right\|^2\\
\le& \frac{1}{N}\sum_{i=1}^N \left(- l^*_i(\alpha''_i) - l^{*'}_i(\alpha''_i)(\alpha'_i - \alpha''_i) - \frac{\mu}{2}(\alpha'_i-\alpha''_i)^2 \right) - \frac{\lambda}{2}\|w(\alpha'')\|^2 + \frac{1}{N}\sum_{i=1}^N x_i^\top w(\alpha'')(\alpha'_i - \alpha''_i) \\
 &- \frac{1}{2\lambda N^2} (\alpha' - \alpha'')^\top X_{F''}^\top X_{F''}(\alpha' - \alpha'')\\
\le& D(\alpha'') + \langle g(\alpha''), \alpha' - \alpha ''\rangle -\frac{\lambda N\mu + \sigma^2_{\min}(X,k)}{2\lambda N^2}\|\alpha' - \alpha''\|^2.
\end{aligned}
\]
This proves the first desirable inequality in the lemma. By invoking the above inequality and using the fact $D(\alpha) \le D(\bar\alpha)$ we get that
\[
\begin{aligned}
D(\bar\alpha) \le& D(\alpha) + \langle g(\alpha), \bar\alpha - \alpha\rangle - \frac{\lambda N \mu + \sigma^2_{\min}(X,k)}{2\lambda N^2}\|\alpha - \bar\alpha\|^2 \\
\le&  D(\bar\alpha) + \langle g(\alpha), \bar\alpha - \alpha\rangle - \frac{\lambda N \mu + \sigma^2_{\min}(X,k)}{2\lambda N^2}\|\alpha - \bar\alpha\|^2,
\end{aligned}
\]
which leads to the second desired bound.
\end{proof}

The following lemma gives a simple expression of the gap for properly related primal-dual pairs.
\begin{lemma}\label{lemma:gap}
Given a dual variable $\alpha \in \mathcal{F}^N$ and the related primal variable
\[
w = \mathrm{H}_k\left(-\frac{1}{N\lambda} \sum_{i=1}^N \alpha_i x_i \right).
\]
The primal-dual gap $\epsilon_{PD}(w,\alpha)$ can be expressed as:
\[
\epsilon_{PD}(w,\alpha) = \frac{1}{N}\sum_{i=1}^N \left(l_i(w^\top x_i) + l_i^*(\alpha_i) - \alpha_i w^\top x_i\right).
\]
\end{lemma}
\begin{proof}
It is directly to know from the definitions of $P(w)$ and $D(\alpha)$ that
\[
\begin{aligned}
&P(w) - D(\alpha) \\
=& \frac{1}{N}\sum_{i=1}^N l_i(w^\top x_i) + \frac{\lambda}{2} \|w\|^2 - \left(\frac{1}{N}\sum_{i=1}^N\left(\alpha_iw^\top x_i - l^*_i(\alpha_i)\right) + \frac{\lambda}{2}\|w\|^2\right) \\
=& \frac{1}{N}\sum_{i=1}^N \left(l_i(w^\top x_i) + l_i^*(\alpha_i) - \alpha_i w^\top x_i\right).
\end{aligned}
\]
This shows the desired expression.
\end{proof}

Based on Lemma~\ref{lemma:gap}, we can derive the following lemma which establishes a bound on the primal-dual gap.
\begin{lemma}\label{lemma:gap_bound}
Consider a primal-dual pair $(w,\alpha)$ satisfying
\[
w = \mathrm{H}_k\left(-\frac{1}{N\lambda} \sum_{i=1}^N \alpha_i x_i \right).
\]
Then the following inequality holds for any $g(\alpha)\in \partial D(\alpha)$ and $\beta \in [\partial l_1(w^\top x_1),...,\partial l_N(w^\top x_N)]$:
\[
P(w) - D(\alpha) \le \langle g(\alpha), \beta - \alpha\rangle.
\]
\end{lemma}
\begin{proof}
For any $i \in [1,...,N]$, from the maximizing argument property of convex conjugate we have
\[
l_i(w^\top x_i) = w^\top x_i l'_i(w^\top x_i) - l^*_i(l'_i(w^\top x_i)),
\]
and
\[
l^*_i(\alpha_i) = \alpha_i l^{*'}_i(\alpha_i) - l_i( l^{*'}_i(\alpha_i)).
\]
By summing both sides of above two equalities we get
\begin{equation}\label{ineqt:conv_DIHT_key1}
\begin{aligned}
&l_i(w^\top x_i) + l^*_i(\alpha_i) \\
 =& w^\top x_i l'_i(w^\top x_i) + \alpha_i l^{*'}_i(\alpha_i) -(l_i(l^{*'}_i(\alpha_i)) + l^*_i(l'_i(w^\top x_i))) \\
\overset{\zeta_1}{\le}& w^\top x_i l'_i(w^\top x_i) + \alpha_i l^{*'}_i(\alpha_i) - l^{*'}_i(\alpha_i) l'_i(w^\top x_i),
\end{aligned}
\end{equation}
where ``$\zeta_1$'' follows from Fenchel-Young inequality. Therefore
\[
\begin{aligned}
&\langle g(\alpha), \beta - \alpha\rangle  \\
 =& \frac{1}{N}\sum_{i=1}^N (w^\top x_i - l^{*'}_i(\alpha_i))(l'_i(w^\top x_i)-\alpha_i)  \\
 =& \frac{1}{N}\sum_{i=1}^N \left( w^\top x_i l'_i(w^\top x_i) - l^{*'}_i(\alpha_i) l'_i(w^\top x_i) -\alpha_i w^\top x_i + \alpha_i l^{*'}_i(\alpha_i)\right) \\
 \overset{\zeta_2}{\ge}& \frac{1}{N}\sum_{i=1}^N (l_i(w^\top x_i) + \alpha_i l^*_i(\alpha_i) -w^\top x_i ) \\
 \overset{\zeta_3}{=}& P(w) - D(\alpha),
\end{aligned}
\]
where ``$\zeta_2$'' follows from~\eqref{ineqt:conv_DIHT_key1} and ``$\zeta_3$'' follows from Lemma~\ref{lemma:gap}. This proves the desired bound.
\end{proof}

The following simple result is also needed in our iteration complexity analysis.
\begin{lemma}\label{lemma:tlnt}
For any $\epsilon>0$,
\[
\frac{1}{t} + \frac{\ln t}{t} \le \epsilon
\]
holds when $t\ge \max\left\{\frac{3}{\epsilon} \ln \frac{3}{\epsilon}, 1\right\}$.
\end{lemma}
\begin{proof}
Obviously, the inequality $\frac{1}{t} + \frac{\ln t}{t} \le \epsilon$ holds for $\epsilon \ge 1$. When $\epsilon < 1$, it holds that $\ln (\frac{3}{\epsilon})\ge 1$. Then the condition on $t$ implies that $\frac{1}{t} \le \frac{\epsilon}{3}$. Also, we have
\[
\frac{\ln t}{t} \le \frac{\ln (\frac{3}{\epsilon}\ln \frac{3}{\epsilon})}{\frac{3}{\epsilon}\ln \frac{3}{\epsilon}} \le \frac{\ln (\frac{3}{\epsilon})^2}{\frac{3}{\epsilon}\ln \frac{3}{\epsilon}} = \frac{2\epsilon}{3},
\]
where the first ``$\le$'' follows the fact that $\ln t/t$ is decreasing when $t\ge 1$ while the second ``$\le$'' follows $\ln x < x$ for all $x>0$. Therefore we have $\frac{1}{t} + \frac{\ln t}{t} \le \epsilon$.
\end{proof}

We are now in the position to prove the main theorem.
\begin{proof}[of Theorem~\ref{thrm:DIHT_Conv}]
\noindent\textbf{Part(a):} Let us consider $g^{(t)} \in \partial D(\alpha^{(t)})$ with $g_i^{(t)}=\frac{1}{N}(x_i^\top w^{(t)}- l^{*'}_i(\alpha_i^{(t)}))$. From the expression of $w^{(t)}$ we can verify that $\|w^{(t)}\|\le r/\lambda$. Therefore we have
\[
\|g^{(t)}\| \le c_0=\frac{r + \lambda\rho}{\lambda\sqrt{N}}.
\]
Let $h^{(t)} = \|\alpha^{(t)} - \bar \alpha\|$ and $v^{(t)} = \langle g^{(t)}, \bar\alpha -\alpha^{(t)}\rangle$. The concavity of $D$ implies $v^{(t)}\ge 0$. From Lemma~\ref{lemma:strong_concavity} we know that $h^{(t)}\le \sqrt{2\lambda N^2 v^{(t)}/(\lambda N \mu + \sigma_{\min}(X,k))}$.  Then
\[
\begin{aligned}
(h^{(t)})^2 =& \|\mathrm{P}_{\mathcal{F^N}}\left(\alpha^{(t-1)}+\eta^{(t-1)} g^{(t-1)}\right) - \bar \alpha\|^2 \\
\le& \|\alpha^{(t-1)} + \eta^{(t-1)} g^{(t-1)} - \bar \alpha\|^2 \\
=& (h^{(t-1)})^2 - 2\eta^{(t-1)} v^{(t-1)} + (\eta^{(t-1)})^2 \|g^{(t-1)}\|^2 \\
\le& (h^{(t-1)})^2 - \frac{\eta^{(t-1)}(\lambda N \mu + \sigma_{\min}(X,k))}{\lambda N^2} (h^{(t-1)})^2 + (\eta^{(t-1)})^2 c_0^2.
\end{aligned}
\]
Let $\eta^{(t)} = \frac{\lambda N^2}{(\lambda N \mu + \sigma_{\min}(X,k))(t+1)}$. Then we obtain
\[
(h^{(t)})^2 \le \left(1-\frac{1}{t}\right) (h^{(t-1)})^2 + \frac{\lambda^2 N^4 c_0^2}{(\lambda N \mu + \sigma_{\min}(X,k))^2t^2}.
\]
By recursively applying the above inequality we get
\[
(h^{(t)})^2 \le \frac{\lambda^2 N^4 c_0^2}{(\lambda N \mu + \sigma_{\min}(X,k))^2} \left(\frac{1}{t} + \frac{\ln t}{t}\right) = c_1 \left(\frac{1}{t} + \frac{\ln t}{t}\right).
\]
This proves the desired bound in part(a).

\noindent\textbf{Part(b):} Let us consider $\epsilon = \frac{\lambda N \bar\epsilon}{2\sigma_{\max}(X)}$. From part(a) and Lemma~\ref{lemma:tlnt} we obtain
\[
\|\alpha^{(t)} - \bar \alpha\| \le \epsilon
\]
after $t \ge t_0 = \frac{3c_1}{\epsilon^2}\ln \frac{3c_1}{\epsilon^2}$. It follows from Lemma~\ref{lemma:uniqueness_gap} that $\supp(w^{(t)}) = \supp(\bar w)$.

Let $\beta^{(t)}:=[l'_1((w^{(t)})^\top x_1),...,l'_N((w^{(t)})^\top x_N)]$. According to Lemma~\ref{lemma:gap_bound} we have
\[
\begin{aligned}
\epsilon_{PD}^{(t)} &= P(w^{(t)}) - D(\alpha^{(t)}) \\
&\le \langle g^{(t)}, \beta^{(t)} - \alpha^{(t)}\rangle \\
& \le \|g^{(t)}\| (\|\beta^{(t)} -\bar \alpha\| + \|\bar\alpha- \alpha^{(t)}\|).
\end{aligned}
\]
Since $\bar\epsilon = \bar w_{\min} - \frac{1}{\lambda} \|P'(\bar w)\|_\infty >0$, it follows from Theorem~\ref{thrm:sparse_mini_max} that $
\bar\alpha = [l'_1(\bar w^\top x_1),..., l'_N(\bar w^\top x_N)]$. Given that $t \ge t_0$,  from the smoothness of $l_i$ and Lemma~\ref{lemma:uniqueness_gap} we get
\[
\|\beta^{(t)} - \bar\alpha\|\le \frac{1}{\mu} \|w^{(t)} - \bar w\| \le \frac{\sigma_{\max}(X,k)}{\mu\lambda N} \|\alpha^{(t)}-\bar\alpha\|,.
\]
where in the first ``$\le$'' we have used $\|x_i\|\le 1$. Therefore, the following is valid when $t \ge t_0$:
\[
\begin{aligned}
\epsilon_{PD}^{(t)} &\le  \|g^{(t)}\| (\|\beta^{(t)} -\bar \alpha\| + \|\bar\alpha- \alpha^{(t)}\|) \\
&\le c_0\left( 1 + \frac{\sigma_{\max}(X,k)}{\mu\lambda N} \right) \|\alpha^{(t)} - \bar\alpha\|.
\end{aligned}
\]
Since $t \ge t_1$, from part(a) and Lemma~\ref{lemma:tlnt} we get $\|\alpha^{(t)} - \bar \alpha\| \le \frac{\epsilon}{c_0\left( 1 + \frac{\sigma_{\max}(X,k)}{\mu\lambda N} \right)}$, which according to the above inequality implies $\epsilon_{PD}^{(t)} \le \epsilon$. This proves the desired bound.
\end{proof}

\subsection{Proof of Theorem~\ref{thrm:SDIHT_Conv}}
\label{append:proof_SDIHT_conv}

\begin{proof}
\noindent\textbf{Part(a):} Let us consider $g^{(t)}$ with $g_j^{(t)}=\frac{1}{N}(x_j^\top w^{(t)}- l^{*'}_j(\alpha_i^{(t)}))$. Let $h^{(t)} = \|\alpha^{(t)} - \bar \alpha\|$ and $v^{(t)} = \langle g^{(t)}, \bar\alpha -\alpha^{(t)}\rangle$. The concavity of $D$ implies $v^{(t)}\ge 0$. From Lemma~\ref{lemma:strong_concavity} we know that $h^{(t)}\le \sqrt{2\lambda N^2 v^{(t)}/(\lambda N \mu + \sigma_{\min}(X,k))}$. Let $g^{(t)}_{B_i}:=\mathrm{H}_{B_i^{(t)}}(g^{(t)})$ and $v_{B_i}^{(t)}:=\langle g^{(t)}_{B_i}, \bar\alpha -\alpha^{(t)}\rangle$ Then
\[
\begin{aligned}
(h^{(t)})^2 =& \|\mathrm{P}_{\mathcal{F^N}}\left(\alpha^{(t-1)}+\eta^{(t-1)} g_{B_i}^{(t-1)}\right) - \bar \alpha\|^2 \\
\le& \|\alpha^{(t-1)} + \eta^{(t-1)} g_{B_i}^{(t-1)} - \bar \alpha\|^2 \\
=& (h^{(t-1)})^2 - 2\eta^{(t-1)} v_{B_i}^{(t-1)} + (\eta^{(t-1)})^2 \|g_{B_i}^{(t-1)}\|^2.
\end{aligned}
\]
By taking conditional expectation (with respect to uniform random block selection, conditioned on $\alpha^{(t-1)}$) on both sides of the above inequality we get
\[
\begin{aligned}
&\mathbb{E}[(h^{(t)})^2\mid \alpha^{(t-1)}] \\
\le& (h^{(t-1)})^2 - \frac{1}{m}\sum_{i=1}^m 2\eta^{(t-1)} v_{B_i}^{(t-1)} + \frac{1}{m}\sum_{i=1}^m(\eta^{(t-1)})^2 \|g_{B_i}^{(t-1)}\|^2 \\
=& (h^{(t-1)})^2 - \frac{2\eta^{(t-1)}}{m} v^{(t-1)} + \frac{(\eta^{(t-1)})^2}{m} \|g^{(t-1)}\|^2 \\
\le& (h^{(t-1)})^2 - \frac{\eta^{(t-1)}(\lambda N \mu + \sigma_{\min}(X,k))}{\lambda m N^2} (h^{(t-1)})^2 + \frac{(\eta^{(t-1)})^2}{m} c_0^2..
\end{aligned}
\]
Let $\eta^{(t)} = \frac{\lambda m N^2}{(\lambda N \mu + \sigma_{\min}(X,k))(t+1)}$. Then we obtain
\[
\mathbb{E}[(h^{(t)})^2\mid \alpha^{(t-1)}] \le \left(1-\frac{1}{t}\right) (h^{(t-1)})^2 + \frac{\lambda^2 m N^4 c_0^2}{(\lambda N \mu + \sigma_{\min}(X,k))^2t^2}.
\]
By taking expectation on both sides of the above over $\alpha^{(t-1)}$, we further get
\[
\mathbb{E}[(h^{(t)})^2] \le \left(1-\frac{1}{t}\right) \mathbb{E}[(h^{(t-1)})^2] + \frac{\lambda^2 m N^4 c_0^2}{(\lambda N \mu + \sigma_{\min}(X,k))^2t^2}.
\]
This recursive inequality leads to
\[
\mathbb{E}[(h^{(t)})^2] \le \frac{\lambda^2 m N^4 c_0^2}{(\lambda N \mu + \sigma_{\min}(X,k))^2} \left(\frac{1}{t} + \frac{\ln t}{t}\right) = c_2 \left(\frac{1}{t} + \frac{\ln t}{t}\right).
\]
This proves the desired bound in part(a).

\noindent\textbf{Part(b):} Let us consider $\epsilon = \frac{\lambda N \bar\epsilon}{2\sigma_{\max}(X)}$. From part(a) and Lemma~\ref{lemma:tlnt} we obtain
\[
\mathbb{E} [\|\alpha^{(t)} - \bar \alpha\|] \le \delta\epsilon
\]
after $t \ge t_2 = \frac{3c_2}{\delta^2\epsilon^2}\ln \frac{3c_2}{\delta^2\epsilon^2}$. Then from Markov inequality we know that $\|\alpha^{(t)} - \bar \alpha\|\le \mathbb{E} [\|\alpha^{(t)} - \bar \alpha\|]/\delta \le \epsilon$ holds with probability at least $1-\delta$. Lemma~\ref{lemma:uniqueness_gap} shows that $\|\alpha^{(t)} - \bar \alpha\|\le \epsilon$ implies $\supp(w^{(t)}) = \supp(\bar w)$. Therefore when $t \ge t_2$, the event $\supp(w^{(t)}) = \supp(\bar w)$ occurs with probability at least $1-\delta$.

Similar to the proof arguments of Theorem~\ref{thrm:DIHT_Conv}(b) we can further show that when $t \ge 4 t_2$, with probability at least $1-\delta/2$
\[
\|\alpha^{(t)} - \bar\alpha\| \le \frac{\lambda N \bar\epsilon}{2\sigma_{\max}(X)},
\]
which then leads to
\[
\epsilon_{PD}^{(t)} \le c_0 \left( 1 + \frac{\sigma_{\max}(X,k)}{\mu\lambda N} \right) \|\alpha^{(t)} - \bar\alpha\|.
\]
Since $t\ge t_3$, from the arguments in part(a) and Lemma~\ref{lemma:tlnt} we get that $\|\alpha^{(t)} - \bar \alpha\| \le \frac{\epsilon}{c_0\left( 1 + \frac{\sigma_{\max}(X,k)}{\mu\lambda N} \right)}$ holds with probability at least $1-\delta/2$. Let us consider the following events:
\begin{itemize}
  \item $\mathcal{A}$: the event of $\epsilon_{PD}^{(t)} \le \epsilon$;
  \item $\mathcal{B}$: the event of $\|\alpha^{(t)} - \bar\alpha\|\le \frac{\lambda N \bar\epsilon}{2\sigma_{\max}(X)}$;
  \item $\mathcal{C}$: the event of $\|\alpha^{(t)} - \bar \alpha\| \le \frac{\epsilon}{c_0\left( 1 + \frac{\sigma_{\max}(X,k)}{\mu\lambda N} \right)}$.
\end{itemize}
When $t\ge \max\{4t_2, t_3\}$, we have the following holds:
\[
\mathbb{P} (\mathcal{A}) \ge \mathbb{P} (\mathcal{A}\mid \mathcal{B}) \mathbb{P} (\mathcal{B})\ge \mathbb{P} (\mathcal{C}\mid \mathcal{B})\mathbb{P}(\mathcal{B}) \ge (1-\delta/2)^2 \ge 1-\delta.
\]
This proves the desired bound.
\end{proof}

\end{document}